\documentclass[twoside,11pt]{article}

\usepackage{jmlr2e}

\usepackage{amsfonts}
\usepackage{mathtools}
\usepackage{etoolbox} 
\usepackage{letltxmacro} 

\usepackage{color}

\newtheorem{thm}{Theorem}[section]

\newtheorem{lem}[thm]{Lemma}

\newcommand{\vtt}[1]{%
  \text{\normalfont\ttfamily\detokenize{#1}}%
}

\DeclareMathOperator*{\argmin}{arg\,min}

\newtoggle{graphics}
\LetLtxMacro{\reallyincludegraphics}{\includegraphics}
\renewcommand{\includegraphics}[2][]{\iftoggle{graphics}{\reallyincludegraphics[#1]{#2}}{}}

\toggletrue{graphics}

\jmlrheading{1}{2019}{1-?}{?}{?}{?}{Veit Elser}

\ShortHeadings{Learning Without Loss}{Elser}
\firstpageno{1}

\begin{document}

\title{Learning Without Loss} 

\author{\name Veit Elser \email ve10@cornell.edu \\
       \addr Department of Physics\\
       Cornell University\\
       Ithaca, NY 14853-2501, USA
      }

\editor{?}

\date{}

\maketitle

\begin{abstract}
We explore a new approach for training neural networks where all loss functions are replaced by hard constraints. The same approach is very successful in phase retrieval, where signals are reconstructed from magnitude constraints and general characteristics (sparsity, support, etc.). Instead of taking gradient steps, the optimizer in the constraint based approach, called relaxed-reflect-reflect (RRR), derives its steps from projections to local constraints. In neural networks one such projection makes the minimal modification to the inputs $x$, the associated weights $w$, and the pre-activation value $y$ at each neuron, to satisfy the equation $x\cdot w=y$. These projections, along with a host of other local projections (constraining pre- and post-activations, etc.) can be partitioned into two sets such that all the projections in each set can be applied concurrently --- across the network \textit{and} across all data in the training batch. This partitioning into two sets is analogous to the situation in phase retrieval and the setting for which the general purpose RRR optimizer was designed. Owing to the novelty of the method, this paper also serves as a self-contained tutorial. Starting with a single-layer network that performs non-negative matrix factorization, and concluding with a generative model comprising an autoencoder and classifier, all applications and their implementations by projections are described in complete detail. Although the new approach has the potential to extend the scope of neural networks (e.g. by defining activation not through functions but constraint sets), most of the featured models are standard to allow comparison with stochastic gradient descent.

\end{abstract}




\section{Introduction}\label{sec:intro}

When general purpose computers arrived in the 1960s it was realized that certain tasks, such as sorting and Fourier transforms, would be so ubiquitous that it made sense to implement them with provably optimal algorithms. In the present day, as neural networks have become ubiquitous in machine learning systems, the optimality of training algorithms has likewise been the subject of intense research. The expressivity of neural networks makes them attractive for diverse applications but is also the origin of their complexity. Even with the simple piecewise-linear ReLU activation function, the number of linear pieces utilized by a network can in principle grow exponentially with the number of neurons. And while there is choice of loss function to apply to the training task, the inherent complexity of the models makes proving optimality, for any loss, well beyond reach.

Faced with the theoretical intractability of neural network training, it is not surprising that research has narrowed on a single empirical strategy: gradient descent. Central to this method of training is a loss function that encapsulates everything relevant to the application, from the definition of class boundaries, to the structure of internal representations, to details such as model sparsity and parameter quantization. Even with the focus of using gradient information to minimize the loss, there are many options (e.g. batch normalization) that need to be evaluated empirically. Advances are, justifiably, incremental. The result is that the theory of neural network training has become a single evolving paradigm.

It would be audacious to propose a fundamentally different approach to neural network training were it not for the fact that there already is an empirically tested alternative with a strong track record. This is the method of optimization that has evolved in the field of \textit{phase retrieval}. Though the analogy is far from perfect, phase retrieval also deals with very large data sets and seeks to discover  representations of data that are meaningful. More significantly, the most successful algorithms for phase retrieval are not based on gradient descent. In this paper we apply these same techniques to the training of neural networks. As with gradient descent there are many options and we present only a particular approach that is both flexible and empirically successful.

Phase retrieval uses the measured magnitudes of a complex-valued signal, together with generic properties (signal support, sparsity) to reconstruct the signal's phases. It is possible to define loss functions for phase retrieval and attempt the discovery of the phases with gradient-based  methods. This approach, called Wirtinger flow \citep{candes2015phase}, has led to a recent revival of interest in the theoretical problem. However, this line of research has not produced any practical algorithms to displace the non-gradient algorithms that are used in applications. Instead of minimizing a loss function, these algorithms try to discover a point $x_\mathrm{sol}$ that lies in the intersection of two sets:
\begin{equation}\label{AintB}
x_\mathrm{sol}\in A\cap B.
\end{equation}
The two sets live in a high dimensional Euclidean space, not unlike the space of parameters for a network trained on images. In phase retrieval $A$ is the set of all images with given (Fourier-transform) magnitudes and $B$ is all images having a particular support or number of atoms (sparsity). The details of these two constraint sets only enter the algorithm through the action of two \textit{constraint projections}, $P_A$ and $P_B$. These provide exact solutions of two global subproblems: for arbitrary $x$, find points $x_A=P_A(x)\in A$ and $x_B=P_B(x)\in B$ on the two constraint sets that are proximal to $x$.

Importing the methodology of phase retrieval to neural networks is mostly about formulating the training problem as an instance of \eqref{AintB} for suitable $A$ and $B$, with the additional property that the corresponding constraint projections can be computed efficiently. Before we preview our approach to this, we describe a general-purpose algorithm for solving \eqref{AintB} and contrast it with gradient descent.

We will train neural networks with the \textit{relaxed-reflect-reflect} (RRR) algorithm, probably the simplest of the algorithms that are successful in phase retrieval. Both RRR and the stochastic gradient descent (SGD) algorithm are iterative with a time-step (learning rate) parameter $\beta$. In the limit $\beta\to 0$ both methods define a flow:
\begin{equation}\label{flow}
\dot{x}=F(x).
\end{equation}
In SGD $x$ is the vector of network parameters and the vector field $F$ is the gradient of the loss (including regularization terms) with respect to those parameters. In the alternative RRR method the vector $x$ in \eqref{flow} also includes the node values, pre- and post-activation, for some number of instantiations of the network. More significantly, the flow field for RRR,
\begin{equation}
F(x)=P_B(2 P_A(x)-x) - P_A(x),
\end{equation}
is not the gradient of any function. Fixed points of the flow, defined by $F(x^*)=0$, are a problem for SGD training because optimization ceases without the guarantee that the loss is zero. For RRR this same condition implies
\begin{equation}
P_B(2 P_A(x^*)-x^*) = P_A(x^*)=x_\mathrm{sol},
\end{equation}
or a solution to \eqref{AintB} because the $x_\mathrm{sol}$ so defined lies in both $A$ and $B$. The relationship between $x^*$ and $x_\mathrm{sol}$ is illustrated in Figure 1 for sets $A$ and $B$ that locally are flats in the ambient space (red and green lines), a model that applies to our use of the algorithm. In a feasible instance (left panel) with unique solution $A\cap B=\{x_\mathrm{sol}\}$, the orbit of $x$ converges to any point $x^*$ in the space orthogonal to $A$ and $B$ at $x_\mathrm{sol}$ (black dashed line). All of these fixed points are associated to the same solution point, $x_\mathrm{sol}$. When the instance is infeasible (right panel) the orbit converges to the same space it did in the feasible instance except that now it also moves uniformly within this space and away from the ``near intersection" of $A$ and $B$. This behavior is also desirable as it locates points $x_A=P_A(x^+)$ and $x_B=P_B(2 P_A(x^+)-x^+)$, associated to the asymptotic orbit $x^+$, that are proximal on the two constraint sets (minimize $\|F\|$). In the neural network setting finding a best approximate solution arises in the training of autoencoders.

\begin{figure}[t]
\begin{center}
\includegraphics[width=6.in]{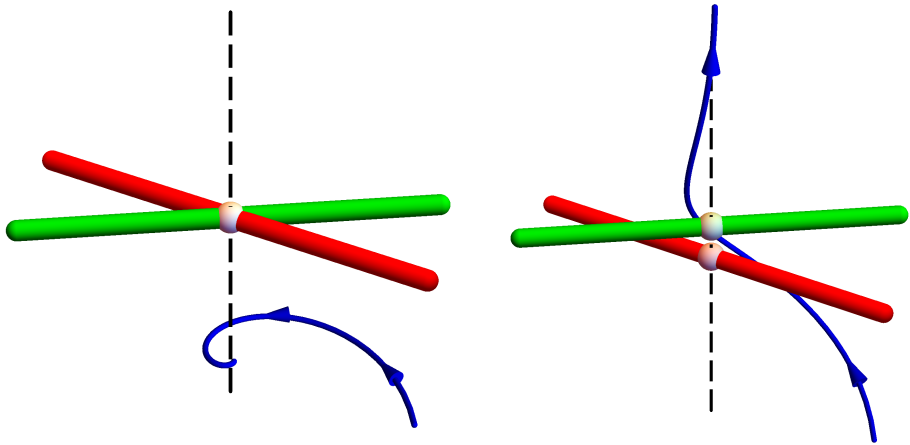}
\end{center}
\caption{RRR orbits (blue) for feasible (left) and infeasible (right) instances.}
\label{fig:fig1}
\end{figure}

An interesting contrast between RRR and gradient descent algorithms is how they manage to avoid getting stuck. In strict gradient descent with loss function $\mathcal{L}$ and flow $F=-\nabla \mathcal{L}$, the loss is monotonically decreasing. The leading strategies for avoiding local minima where $\mathcal{L}>0$ are (i) the stochastic estimation of $F$ (SGD algorithm), thereby relaxing monotonicity, and (ii) exploiting application-specific structure when initializing the flow \citep{candes2015phase}. A third strategy, having special relevance to machine learning, is to use SGD only for highly over-parametrized models, where the loss landscape may be free of traps.

The speed of the flow, $v=\|F\|$, is weakly analogous to loss for the RRR algorithm. Like a restless shark, RRR keeps moving while $v>0$, ceasing only when the solution ($v=\|F\|=0$) is in its maw. However, unlike $\mathcal{L}$ in gradient descent, $v$ does not decrease monotonically under RRR evolution. An easy way to show that the RRR flow field $F$ cannot be the gradient of a function is to construct toy examples (particular sets $A$ and $B$) in two dimensions where the flow has limit-cycle behavior. Limit cycles, if abundant, pose a possible trapping mechanism for RRR, analogous to the local minima faced by SGD. Empirically we have seen very little evidence of limit-cycle trapping, even in challenging small models  where gradient methods consistently fail because of the local minimum problem.

There are roughly three kinds of constraints that define the sets $A$ and $B$ in our application of RRR to neural network training. These are constraints associated with (i) neuron inputs, (ii) neuron outputs, and (iii) ``consensus" for replicated variables. The input constraints, one for each neuron, have the form
\begin{equation}\label{inputcon}
x\cdot w=y,
\end{equation}
where $w$ and $x$ are the vectors of weights and outputs of other neurons that combine to give the neuron's pre-activation value $y$. We consign this constraint to set $B$, so that $P_B$ applied to an arbitrary $(w,x,y)$ finds the distance minimizing change $(w,x,y)\to (w_B,x_B,y_B)$ that satisfies \eqref{inputcon}. We see that weights $w$ and neuron outputs $x$ are treated more on an equal footing than they are in SGD training, where changes to the neuron outputs appear only implicitly in the backpropagation computations. From each neuron's perspective, changes to $w$ are forward-looking (striving to fix things in higher layers), while changes to $x$ do the opposite (backward-looking, forcing changes in lower layers).

The need for consensus constraints becomes clear when we want to be able to project to the input constraints \eqref{inputcon} independently for each neuron and each data instantiation of the network. A more explicit rewriting of \eqref{inputcon} that makes this possible is
\begin{equation}\label{neuronin}
\forall\; k,j\colon \sum_i x[k,i\to j]\;w[k,i\to j]=y[k,j],
\end{equation}
where $k$ is a data index (network instantiation), $j$ labels the receiving neuron, and the sum is over neurons $i$ whose outputs $x[k,i\to j]$ are incident on $j$. Clearly a neuron $i$ should not be allowed to take different values depending on which neuron $j$ is receiving its value. To insure that this does not happen we impose the following consensus constraints:
\begin{equation}\label{consensx}
\forall\; k,i,j\colon x[k,i\to j]=x_A[k,i].
\end{equation}
Here $x_A[k,i]$ is the consensus value and the subscript indicates it is associated with the set $A$. Similarly, because the weight on edge $i\to j$ should not be allowed to take different values for each data item $k$ we impose another consensus constraint:
\begin{equation}
\forall\; k,i,j\colon w[k,i\to j]=w_A[i\to j].
\end{equation}
By consigning these consensus constraints to set $A$ which is inactive when we perform $P_B$ in the RRR algorithm, the neuron-input constraint  \eqref{neuronin} of set $B$ is able to work with local (otherwise unconstrained) variables.

The third type of constraint implements the activation function $f$ that connects the pre- and post-activation neuron values:
\begin{equation}\label{actcon}
\forall\; k,j\colon f\left(y[k,j]-b[k,j]\right)=x_A[k,j].
\end{equation}
By consigning this to the $A$ constraints, all the $y$ variables, both in \eqref{neuronin} and \eqref{actcon}, are local (within sets $B$ and  $A$, respectively). As with the weights we have had to replicate the bias parameters $b$ to make the constraint local to each data item $k$. Consensus is now imposed in set $B$ (as they are allowed to be independent in constraint $A$):
\begin{equation}
\forall\; k,j\colon b[k,j]=b_B[j].
\end{equation}

Depending on the neural network application (classifier, autoencoder, etc.) there may be modifications to \eqref{consensx} and \eqref{actcon} at input-layer, output-layer or code-layer neurons. There may also be constraints for parameter regularization. We describe these in detail, and their membership in $A$ or $B$, in later sections where we study specific applications.

\begin{figure}[t]
\begin{center}
\includegraphics[width=3.in]{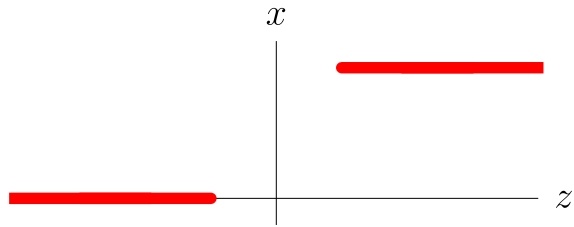}
\end{center}
\caption{Step activation function reinterpreted as a constraint set (red).}
\label{fig:fig2}
\end{figure}

Now that the components of the phase retrieval inspired approach to neural network training have been introduced, we can highlight some of the features that make it attractive. These derive from the greater power that projections potentially have to offer over the gradient moves that underlie nearly all current algorithms. Consider the activation function constraint \eqref{actcon}. Not only is $f$ not required to be a piecewise regular function, it need not even be a function! A good example is the modification of the step ``function" shown in Figure 2. Interpreted as a locus of pre/post-activation pairs $(z=y-b,x)$, a projection $(z,x)\to (z_A,x_A)$ is easily computed even if there is a gap in the domain of the function definition as shown. Such a gap is analogous to the margin parameter normally reserved for the output layer in a classifier. By introducing a margin in the activation, individual neurons will be forced to make unambiguous choices during training, a feature that may improve generalization. After training, if a pre-activation $z$ falls within the gap, $f$ would be interpreted as the standard step function with discontinuity at $z=0$.

Another advantage of projections is that they provide a direct mechanism for imposing structural properties that gradient methods must do indirectly, and with no guarantees, via terms in the loss function. An example of this arises in non-negative matrix factorization, when the learned non-negative feature vectors are also required to be $s$-sparse. Gradient-based methods would introduce a differentiable, sparsity-promoting regularizer such as the 1-norm on the weights. By contrast, projection to a sparsity constraint is direct and in fact much used in phase retrieval: all negative components of the feature vector are set to zero and the remaining positive components are sorted and all but the $s$ largest are also set to zero.

A more elaborate example of the kind of constraint just described, included in our survey of applications, might arise when we suspect some fraction $p$ of the data in supervised learning has wrong labels. In this case we would choose to train on rather large batches, each having say 1000 items. For each item in the batch the projection to the correct-class-constraint at the output layer would be computed both on the assumption that the label is correct and also for the case that the label is wrong. Both hypotheses have a projection distance, and the distance-increases to the stronger (``label is correct") constraint are sorted for the 1000 data. The wrong-label hypothesis/projection option would be applied to the $p\times 1000$ data having the greatest distance increases. 

Because the two projections at the core of RRR can leverage the power of many standard data analysis algorithms (e.g. SVD for projecting to a low-rank constraint), the scope of RRR in machine learning is potentially very broad. However, in this study our focus is relatively narrow: a scheme for training standard network models based on the particular, neuron-centric constraint \eqref{inputcon}. The method is developed through a series of models of increasing sophistication, from non-negative matrix factorization to representation learning. By using a natural warm-start procedure on batches, and varying the batch size, we explore both the on-line case of small batches and the off-line mode where projections are applied potentially to the entire data set. Projections take the place of back-propagation and have a comparable operation count. Each parameter update (iteration) of RRR has about as many operations as one SGD step, scaling as the product of the size of the batch and the number of edges in the network.

To help build the case that constraint-based, loss-free optimization could serve as a low-level computational framework for training, all the software for our numerical experiments was implemented with $C$ programs that only call the standard $C$ libraries. Our code runs serially, on a single thread, and without GPU acceleration. Thanks to the ``split" nature of the constraints, parallelization could easily have been introduced and indeed it is this feature that partly motivated the related work that we review next.

\section{Relationship to prior work}\label{sec:ADMM}

The idea of demoting the pre- and post-activation states of the neurons, from known values determined by forward propagation of the data, to variables that have to be solved along with the network parameters, is not new and has been explored by several groups. Central to this strategy is a scheme for splitting the augmented set of variables into groups amenable to exact, local optimization. For neural network training the two most popular named methods, for acting on the split variables, are alternating direction method of multipliers (ADMM) and block coordinate descent (BCD). In an influential study,  \cite{taylor2016training} used ADMM in networks split along layers, where optimization alternates between the weights of individual layers and the activations that join them, in rough correspondence with, respectively, our constraints \eqref{neuronin} and \eqref{actcon}. By also splitting with respect to data items, \cite{taylor2016training} are able to train on the entire data in aggregate (not sequentially). More recently, \cite{choromanska2018beyond} have compared the ADMM and BCD variants, also using layer-wise network splitting, but where data are processed individually, more in the style of SGD.

Below is a summary of the ways in which our work differs from, or goes beyond, previous work:
\begin{itemize}
\item No loss functions are used.
\item Networks are split on a finer scale: neurons rather than layers.
\item Our optimizer, RRR, is entirely built from projections.
\item Through warm-start initialization we can train on batches of arbitrary size, up to the entire data set.
\item Variations of the same framework are shown to perform non-negative matrix factorization, classification, and representation learning.
\item The flexibility of the constraint formulation allows us to build a ``relabeling classifier" and a non-adversarial generative model.
\item We demonstrate learning on hard, small models where gradient methods fail.
\end{itemize}
The RRR algorithm, with its projections as computational primitives, has a closer relationship to ADMM than it might seem. First, when using indicator functions for the two objective functions in the ADMM formalism, the minimization steps in ADMM reduce to projections. Moreover, in this projection setting the ``unrelaxed" ADMM iteration turns out to be exactly equivalent to the RRR iteration with time-step $\beta=1$. On the other hand, the most direct way to understand why RRR works at all is to consider $\beta\to 0$ (Figure \ref{fig:fig1}) where RRR and ADMM differ. The analysis of this limit and details on the RRR/ADMM relationship are given in appendix \ref{sec:RRRappendix}.

\section{Organization and notation}\label{sec:notation}

This paper is written in the style of a tutorial. Unsupervised training in the loss-free style of learning is provided through a series of examples. The examples are ordered with increasing complexity and correspond to the three types of network shown in Figure \ref{fig:fig3}. Readers only interested in deep networks should begin with the first of these, on non-negative matrix factorization (NMF), as many of the same elements are used even in this simplest case. Proofs of mathematical results are placed in the appendix for readability.

\begin{figure}[t]
\begin{center}
\includegraphics[width=5.5in]{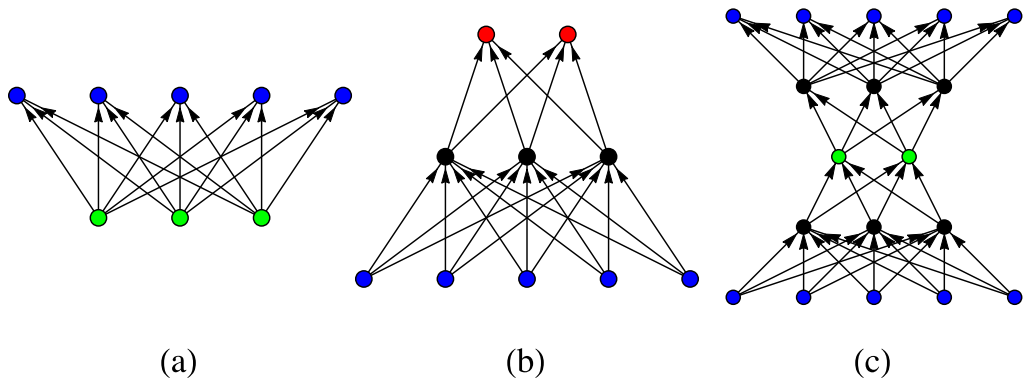}
\end{center}
\caption{The three types of network featured in this tutorial: (a) single layer network for non-negative matrix factorization, (b) multi-layer classifier network, (c) autoencoder network. Networks are not required to be layered, except for special sets of nodes : data layer (blue), code layer (green), class layer (red). In the autoencoder network the single data layer is rendered twice (top and bottom) and the network is cyclic.}
\label{fig:fig3}
\end{figure}

We use a uniform set of notational conventions in all the examples. Variables and parameters live on directed graphs, even in the NMF application. In our neuron-centric scheme the notion of layers arises only in the designation of the network inputs and outputs; in the autoencoder example there is also a code layer. Indices $i$ and $j$ label nodes and $i\to j$ is the label for the edge from $i$ to $j$. The index $k$ is reserved for the data item label. Inference, or feed-forward processing in the network, is always defined by the equations
\begin{subequations}\label{feedforward}
\begin{align}
y[k,j]&= \sum_i x[k,i]\;w[i\to j]\\
x[k,j]&=f(y[k,j]-b[j]),
\end{align}
\end{subequations}
where $w$ are the weights of the network and $b$ the bias parameters. The variables $x$ and $y$ will always be post- and pre-activation values of the nodes; the bias and activation function $f$ that relates them is absent in NMF. Note that the $x$ in the general discussion of the RRR algorithm (section \ref{sec:intro} and appendix) is a search vector that includes all the variables and parameters in the optimization, not just the node values. On an acyclic graph the nodes can be labeled such that $i<j$ for every edge $i\to j$ and inference is well defined by evaluating \eqref{feedforward} for $j$ increasing sequentially. In our constraint-based training even the acyclic property can be relaxed. All of our applications feature layered networks; networks with simple cycles appear in the autoencoder example.

Upper case symbols denote sets, such as the geometrical constraint sets $A$ and $B$ of RRR. We use $K$ for the set of data items in a batch. $E$ is always the set of edges, $D$ the set of \textit{data} nodes, and $C$ is the set of \textit{class} nodes in a classifier or \textit{code} nodes in an autoencoder. The cardinality of discrete sets is indicated by vertical bars, as in the number of edges $|E|$, and $\|\cdots\|$ is always the Euclidean norm.

Greek symbols are reserved for hyperparameters. RRR has a time-step parameter $\beta$ and the parameter $\gamma$ introduced in the appendix. While the flow limit $\beta\to 0$ is the easiest to understand, in most of our experiments we use the Douglas-Rachford value $\beta=1$. The effect of changing $\gamma$ from the standard choice $\gamma=1$ has not been explored. Margins, both at the output of a classifier and for step-activation, are parameterized by $\Delta$. All weights are normalized with norm $\Omega$.

\section{Non-negative matrix factorization}\label{sec:NMF}

In non-negative matrix factorization (NMF) one tries to express non-negative data vectors $y_1,y_2,\ldots$ as non-negative mixtures of a set of non-negative feature vectors $w_1,w_2,\ldots, w_n$. If the data vectors have length $m$, then in terms of the $m\times n$ matrix of feature vectors $W$, we seek a representation of the data as $y_1= W x_1$, $y_2=W x_2$, ... where the mixture vectors $x_1, x_2,\ldots$ have length $n$. When there are $|K|$ data, arranging the data vectors and mixture vectors into matrices as well, the $|K|$ representations take the form of the factorization $Y=W X$.

If there exists a factorization where both $W$ and $X$ have no zeros, then it is highly non-unique because $W'=W A$ and $X'=A^{-1}X$ gives another NMF for an arbitrary $n\times n$ matrix $A$ suitably close to the identity. This $n^2$-dimensional family of factorizations collapses to a much smaller one, and NMF becomes considerably harder, when $W$ or $X$ or both have zeros. In general, including this case, a NMF is always non-unique with respect to permuting and scaling the $n$ columns of $W$. We will take advantage of this freedom, and also compactify the NMF problem, by insisting that the columns of $W$ have a fixed norm of our choosing.

In ``exact NMF" the minimum number of feature vectors in a factorization of the data matrix $Y$ is its positive rank, $r_+(Y)$. Since the positive rank is lower-bounded by the ordinary rank, and real-world data matrices are usually full rank, finding an exact NMF is not the goal in most data science applications. Instead, one specifies the number of feature vectors $n$ and seeks an approximate factorization $Y\approx W X$ with $n$ as the rank. The RRR algorithm can be used for this case too, through its ability to approximately satisfy constraints.

In the online setting, NMF is a one-layer neural network comprising a layer of mixture or \textit{code} nodes $C$ that feed forward to a layer of data nodes $D$ (Figure \ref{fig:fig3}(a)). The network weights $w[i\to j]$ that connect code nodes $i\in C$ to data nodes $j\in D$ are the feature vectors of NMF. While there are no network inputs in the usual sense, training is still defined on data batches $K$, if somewhat indirectly. The network is tasked to learn weights $w$ such that for each data vector $y[k]$, $k\in K$, there exists a corresponding non-negative code vector $x[k]$ that when fed through the network gives a close approximation to $y[k]$. So in addition to learning weights, the network must also learn the code that goes with each data item in the batch.

Provided only that the matrix of weights $W$ in an approximate factorization $Y\approx W X$ of a batch is full rank, an encoder can be constructed starting with the standard pseudo-inverse
\begin{equation}\label{Winv}
\mathcal{E}_W=(W^T W)^{-1}W^T.
\end{equation}
This $n\times m$ matrix may be interpreted as the weights in the encoding stage of an autoencoder, and should be followed by a ReLU activation correction to ensure the codes $x$ are non-negative:
\begin{equation}\label{NMFencode}
x=\mathrm{ReLU}(\mathcal{E}_W\,y).
\end{equation}
We emphasize that this encoding stage only plays a very small part in the training. For modestly sized data sets, where we can process all the data  as a single batch, it plays no role at all. For larger data sets, where the data must be partitioned into batches, we use \eqref{NMFencode} only in the initialization of the RRR solution process (code vectors $x$ initialized for new data $y$ to values determined by weights learned in the previous batch).

\subsection{Constraints}

The variables in our constraint formulation of NMF follow the same pattern we will use in all the other examples: a weight $w[k,i\to j]$ and node variable $x[k,i\to j]$ for each data item $k$ and edge $i\to j$ of the network. In NMF these are the only variables. The node variables $x$ will hold the code vectors and have been replicated, on the edges incident to each code node, in order to split the constraints into independent sets $A$ and $B$. The latter are listed below:
\begin{subequations}
\begin{align}
\mbox{$A$ constraints}\qquad\qquad\qquad&\nonumber\\
\forall\,k\in K, i\in C, j\in D&:\quad x[k,i\to j]=x_A[k,i]\ge 0\label{NMFA1}\\
\forall\,k\in K, i\in C, j\in D&:\quad w[k,i\to j]=w_A[i\to j]\ge 0\label{NMFA2}\\
\forall\,i\in C&:\quad \sum_j w_A^2[i\to j]=\Omega^2\label{NMFA3}\\
\mbox{$B$ constraints}\qquad\qquad\qquad&\nonumber\\
\forall\,k\in K,j\in D&:\quad \sum_i x[k,i\to j]\,w[k,i\to j]=y[k,j].\label{NMFB4}
\end{align}
\end{subequations}
Note that $x_A$ and $w_A$ are not variables but shorthand for consensus values in the constraint. The reason for making the weight norm $\Omega$ a hyperparameter, and not arbitrarily setting it to 1, is discussed below.

Two things need to be checked for any constraint scheme, such as the one above. The first is that the satisfaction of all constraints, in both sets, solves the original problem. This is trivial, as the constraints are simply a transcription of the matrix equation $W X=Y$ with non-negativity constraints on the matrix elements and an additional constraint on the column norms of $W$. The second thing to check is that the constraints in each group, $A$ and $B$ separately, are sufficiently local that it is easy to satisfy them. This seems plausible, but the true test of this comes when we write algorithms for the two constraint projections.

\subsection{Projections}\label{sec:NMFproj}

Projections minimize the Euclidean distance to the constraint set. When the variables come in different varieties --- weights, code vectors --- one should consider applying different distance weighting to the different types of variables. A sufficiently general metric for our NMF scheme would be
\begin{equation}\label{metric}
d^2(x,w)=\sum_{{k\in K}\atop{i\to j\in E}} (x'[k,i\to j]-x[k,i\to j])^2+g^2\, (w'[k,i\to j]-w[k,i\to j])^2,
\end{equation}
where $g^2$ controls the relative compliance of the two types of variables when satisfying constraints. But the rescalings $x\to \sqrt{g}\,x$, $w\to w/\sqrt{g}$ restore the isotropic metric without changing any of the constraints except the value of the norm $\Omega$. The hyperparameter $\Omega$ therefore provides all the advantage one might gain from a parametrized metric and we are free to set $g=1$. Since increasing $g$ was equivalent to increasing $\Omega$, we should set a large value of $\Omega$ when we want the weights to be less compliant than the code vectors.

Constraints \eqref{NMFA1} and also the combination of \eqref{NMFA2} and \eqref{NMFA3} are a common form of \textit{compound constraint}, where the projection seeks a consensus value that additionally satisfies a side constraint. Projections to this type of constraint make use of the following lemma, whose most general form we will need in the later sections:
\begin{lem}\label{compoundconstraint}
Consider variables $\{z_i\in \mathbb{R}^m: i\in I\}$ and $v\in \mathbb{R}^n$  subject to the constraints
\begin{align*}
\forall\, i\in I :&\quad z_i=\overline{z}\\
&\quad(\overline{z},v)\in S,
\end{align*}
where $S\subset \mathbb{R}^{m+n}$ is a set that specifies a side constraint. The projection to this constraint, minimizing
\[
d^2(z_i: i\in I,v)=\sum_{i\in I}\|z'_i-z_i\|^2+g^2\,\|v'-v\|^2,
\]
is given by $z_i\to \overline{z}'$: $i\in I$, $v\to v'$, where
\[
(\overline{z}',v')=P_S\left(\overline{z},v\right),
\]
\[
\overline{z}=\frac{1}{|I|}\sum_{i\in I}z_i,
\]
and $P_S$ is the projection to $S$ minimizing
\[
d^2(\overline{z},v)=|I|\,\|\overline{z}'-\overline{z}\|^2+g^2\,\|v'-v\|^2.
\]
\end{lem}
\begin{proof}
The proof is an elementary exercise in completing the square.
\end{proof}
In plain terms the lemma states that, in the case of a set of variables subject to a constrained consensus constraint, the projection is performed in two stages. In the first stage a consensus value is obtained by a simple average. This is followed, when there is a side constraint, by projecting the average value to the side constraint taking care to weight its distance by the number of variables taking part in the consensus. When the side constraint involves no additional variables the weighting of the distance plays no role.

Thanks to the lemma, projecting to the $A$ constraint is very easy. To project to the non-negativity side constraint one simply sets to zero all the negative weights or code values. When the norm is also fixed in the side constraint, as in \eqref{NMFA3}, rescaling to the correct norm follows the zeroing of the negative elements. There is a slight complication in the case where the consensus weights are all negative, so that the all-zero vector after non-negativity projection cannot be rescaled. The correct projection in this case is to replace the least negative consensus weight by $\Omega$ and set the rest to zero. Detailed  proofs of these statements are given in \cite{bauschke2018projecting}.

The cost of the projection to the $A$ constraint is dominated by the computation of the consensus values of the variables, not the projections to the side constraints, and therefore scales as the product of the batch size and the number of edges in the network, or $|K||E|$.

The constraints in $B$ are independent for all $k\in K$, $j\in D$ and have the form
\begin{equation}\label{bilinear}
x\cdot w = y,
\end{equation}
where $y$ is fixed by the data. This \textit{bilinear constraint} is generalized, in the examples with deep networks to follow, in that $y$ also becomes a variable. Projections to this constraint (and generalization) are the core of our training algorithm. We are not aware of prior uses of this constraint and its projection. The mathematics supporting our description (below) of the projection is given in appendix 
\ref{sec:projappendix}.

The first step of the projection is the computation of two scalars:
\begin{equation}\label{pq}
p=x\cdot w,\qquad q=x\cdot x +w\cdot w.
\end{equation}
We will assume that $x\pm w=0$ never arises in the course of training, so that
\begin{equation*}
0<\|x\pm w\|^2=q\pm 2p,
\end{equation*}
implies that
\begin{equation}\label{q>2p}
q>2|p|
\end{equation}
always holds. The projection $(x,w)\to (x',w')$ is then unique and given in terms of the unique root $u_0\in(-1,1)$ of a rational equation $h_0(u)=0$ derived in appendix \ref{sec:projappendix}:
\begin{subequations}\label{x'w'}
\begin{align}
x'&=\frac{1}{1-u_0^2}(x+u_0 w)\\
w'&=\frac{1}{1-u_0^2}(w+u_0 x).
\end{align}
\end{subequations}
Fixing the precision of the root $u_0$, the cost of the projection is dominated by the arithmetic in \eqref{pq} and \eqref{x'w'} and scales as the lengths of the vectors $x$ and $w$. As a result, the projection to the $B$ constraint scales in the same way as the projection to the $A$ constraint, as $|K||E|$.

\subsection{Training}

Apart from the absence of a loss function, training networks with RRR is in practice not that different from training with gradient methods. In the case of NMF, the network architecture, Figure \ref{fig:fig3}(a), is fixed by the problem. Sparsity of the features or mixtures could be introduced through a modification of the $A$ constraint and corresponding projection (section \ref{sec:NMFproj}). When the data being factorized is large, a choice must be made for the batch size $|K|$. As a general rule, training is improved with larger batches, and if memory is not a factor, it is usually best not to break up the data at all.

In the RRR applications the author is most familiar with, where the solution or near-solution is unique up to symmetries, RRR is remarkably insensitive to initialization. Variables are normally given random initial values to avoid bias and also to build confidence in a solution's uniqueness when it is obtained multiple times. In the over-parameterized setting of neural network models, where solutions normally are far from unique, initialization might turn out to be important. In any case, we next describe an initialization procedure designed to work with the warm starts that must be used when the data is processed in batches.

For NMF, the only randomness we use in the initialization is on the consensus weights $w_A$ that appear in the $A$ constraint. These we sample from the uniform distribution on $[0,1]$, followed by a rescaling to satisfy \eqref{NMFA3}. Recall that the $w$'s (not $w_A$) are the actual RRR variables, and we initialize them as
\begin{equation}\label{winit}
\forall\,k\in K, i\in C, j\in D:\quad w[k,i\to j]\leftarrow w_A[i\to j].
\end{equation}
By construction, these $w$'s satisfy the $A$ constraint. The randomly generated consensus weights $w_A$ are also used to compute the pseudo-inverse \eqref{Winv} used by the encoder \eqref{NMFencode}. With this encoder we obtain the non-negative consensus code vector $x_A$ for each data vector $y$ in our batch. These code vectors then initialize the $x$ variables:
\begin{equation}\label{xinit}
\forall\,k\in K, i\in C, j\in D:\quad x[k,i\to j]\leftarrow x_A[k,i].
\end{equation}
All of our initial RRR variables ($w$ and $x$) thus satisfy the $A$ constraint.

When data is processed in batches we use random initialization only for the first batch. For subsequent batches we use the final (best optimized) $w_A$ from the current batch, and the corresponding pseudo-inverse encoder \eqref{NMFencode} in the initializations \eqref{winit} and \eqref{xinit}. The RRR search variables $w$ and $x$ thereby inherit information from the previous batches by exactly satisfying the $A$ constraints derived from the current-best weights. This method of batch initialization generalizes to our other applications/networks and will be called ``warm start".

In the NMF application there are two ``errors" of interest. Closest to the operation of RRR is the speed of the flow $v$ (section \ref{sec:intro}), or the constrain incompatibility, which vanishes at a solution fixed point. Our normalization convention for this error is
\begin{equation}\label{RRRerr}
(\verb+RRR_err+)^2 =\frac{1}{|K|}\sum_{k\in K}\;\sum_{i\in C,\, j\in D}(w_A[i\to j]-w_B[k,i\to j])^2+(x_A[k,i]-x_B[k,i\to j])^2.
\end{equation}
Closer to the NMF application is the reconstruction error:
\begin{equation}\label{reconerr}
(\verb+recon_err+)^2 =\frac{1}{|K||D|}\sum_{k\in K}\;\sum_{j\in D}\left(y[k,j]-\sum_{i\in C}x_A[k,i]\,w_A[i\to j]\right)^2.
\end{equation}
This identifies $w_A$ and $x_A$ as the actual NMF solution (or near-solution). In the batched setting, also called ``online NMF", the solution code vectors $x_A[k,i]$ for the entire data set are defined by the encoding \eqref{NMFencode} that uses the pseudo-inverse \eqref{Winv} derived from the current $w_A$, and \eqref{reconerr} is evaluated with $K$ as the entire data set. Note that \verb+recon_err+ corresponds to a root-mean-square signal error\footnote{This deviates from the more common ``mean-square error" which is not as directly interpretable.}. The time-series of \verb+RRR_err+ for the RRR iterates is a useful diagnostic for the RRR solution process. Normally we report just the final value of \verb+recon_err+, or its value after each pass through the entire data, also called an ``epoch".

RRR iterations are terminated either by imposing a fixed cutoff, \verb+RRR_iter+, or as soon as \verb+RRR_err+ falls below a small value, \verb+tol+. In batch mode it often makes sense to use both. A small error target \verb+tol+ for the early batches might unfairly emphasize an unrepresentative sampling of the data (overfitting), especially if it requires many iterations to achieve that error. This is avoided by setting a modest \verb+RRR_iter+. Later, after all the data has been seen (some number of epochs) and \verb+RRR_err+ has dropped to where far fewer iterations are needed to reach the target \verb+tol+, a cap on the number of iterations will grow increasingly unnecessary.

To quantify the work performed in finding a factorization, classification, etc., we use an energy-based, parallelization-independent measure. The scaling of the operation count per iteration $|K||E|$, discussed in section \ref{sec:NMFproj}, will continue to hold in the other applications. We therefore define work, or \textit{giga-weight-multiplies}, by
\begin{equation}
\verb+GWMs+ = 10^{-9}\times\verb+iter_count+\times |K||E|,
\end{equation}
where \verb+iter_count+ is the net number of RRR iterations performed over all batches. For our serial C implementations without GPU of the various training algorithms reported in this study, the wall-clock time in seconds is approximately $100\times \verb+GWMs+$.

For NMF there are only two hyperparameters: the step size $\beta$ and the weight norm $\Omega$. In most of our experiments we use the Douglas-Rachford/ADMM step size $\beta=1$ (appendix \ref{sec:RRRappendix}). Local convergence, in the convex case, holds even for this ``large" step size, and from that perspective nothing is gained by making smaller steps. However, very challenging non-convex constraint satisfaction problems, e.g. bit retrieval \citep{elser2018complexity}, are helped with $\beta<1$ and we will take advantage of this in one of the experiments. Our method for selecting $\Omega$ is strictly empirical. Performance degrades both for very small and very large $\Omega$, consistent with the idea that neither factor, $W$ or $X$, should dominate the factorization.

\subsection{Experiments}

The C programs and data sets used in our experiments are publicly available\footnotemark[\value{footnote}]. For the NMF experiments we used the program \texttt{RRRnmf.c}. In our comparisons, here and in the next sections, we use various programs from the package \texttt{scikit-learn}. This package provides coordinate-descent (CD) and multiplicative-update (MU) solvers for NMF, and comes with various options for initializing the factors. We did not make use of the regularization features in these solvers to keep the comparison fair.

\footnotetext{\texttt{github.com/veitelser/LWL}}

\subsubsection{Linear Euclidean Distance Matrices}\label{sec:ledm}

The linear Euclidean distance matrices (LEDMs) are a standard benchmark for exact NMF. Instances are specified by the $m\times m$ matrices
\begin{equation}
Y_{i k}(m)=\left(\frac{i-k}{m-1}\right)^2,\qquad i\in\{1,\ldots,m\},\; k\in\{1,\ldots,m\}.
\end{equation}
These have ordinary rank 3 (for $m\ge 3$) and non-negative rank $r_+$ that grows logarithmically with $m$ \citep{hrubevs2012nonnegative}. The first non-trivial case for NMF, where $r_+(Y(m))<m$, is $m=6$ for which $r_+=5$. Currently the true non-negative rank is known only up to $m=16$. The CD  algorithm is the better solver for these instances but manages to solve the easiest $m=6$ instance in only 1.6\% of attempts from random starts. \cite{vandaele2016heuristics} report getting solution rates up to 80\% on this instance with the hierarchical-alternating-least-squares algorithm of \cite{cichocki2007hierarchical}. However, even this algorithm was not able to solve LEDM instances beyond $m=8$. The heuristic algorithms that have demonstrated a positive solution rate for instances up to $m=16$ use significantly more randomness than just sampling random starting points, for example, by repeatedly randomizing rank-1 terms followed by local convex minimizations. The state-of-the-art is reviewed by \cite{vandaele2016heuristics}.

\begin{table}[t]
\begin{center}
\begin{tabular}{|rr||rrrrr|rr|rr|}
\hline
& & \multicolumn{5}{c|}{RRR} & \multicolumn{2}{c|}{CD}& \multicolumn{2}{c|}{heuristic$^1$}\\
$m$ & $r_+$ & $\beta$ & $\Omega$ & rate & sec & \verb+GWMs+ & rate & sec & rate & sec \\
\hline
6 & 5 & 0.3 & 0.6 & 100\% & 2 & 0.02 & 1.6\% & 100 &100\% & 19 \\
8 & 6 & 0.3 & 0.7 & 100\% & 59 & 0.61 & $<1\%$ & --- & 99\% & 64 \\
12 & 7 & 0.3 & 0.8 & 100\% & 567 & 6.27 & $<1\%$ & --- & 69\% & 37 \\
16 & 8 & 0.3 & 0.9 & 100\% & 3677 & 41.52 & $<1\%$ & --- & 48\% & 104 \\
\hline
\multicolumn{7}{c}{} & \multicolumn{2}{c}{} & \multicolumn{2}{c}{\footnotesize 1) Vandaele et al.}\\
\end{tabular}
\end{center}
\caption{NMF results for RRR on the LEDM instances compared with coordinate descent (CD) and the leading heuristic from \cite{vandaele2016heuristics}. The RRR results are based on 20 runs for each instance.}
\label{tab:tab1}
\end{table}

It appears the RRR algorithm can factor the LEDM instances up to $m=16$ with a 100\% success rate, that is, reliably from a single random start. This is for the hyperparameter settings given in Table \ref{tab:tab1}. RRR relies on randomness too, but not in the usual sense where randomness is injected by hand at some fixed rate, but through the dynamics of the RRR flow which is chaotic for the constraint sets of this particular application. The chaos is reflected in the behavior of \verb+RRR_err+, shown in Figure \ref{fig:fig4} for three runs of the algorithm on the $m=6$ instance. In each run a chaotic searching period is followed by the convergent behavior that applies when, locally, only convex parts of the constraint sets $A$ and $B$ are active. The chaotic/searching period dominates the solution process in hard instances. One might reasonably claim that these transitions from searching to convergent behavior are the closest any neural network has yet come to experiencing an ``aha moment".

The comparison in Table \ref{tab:tab1} shows that what the heuristic algorithm lacks in reliability it easily makes up for in terms of speed. Still, it is interesting that RRR, a deterministic algorithm, is able to solve these hard problems. 

\begin{figure}[t!]
\begin{center}
\includegraphics[width=5.in]{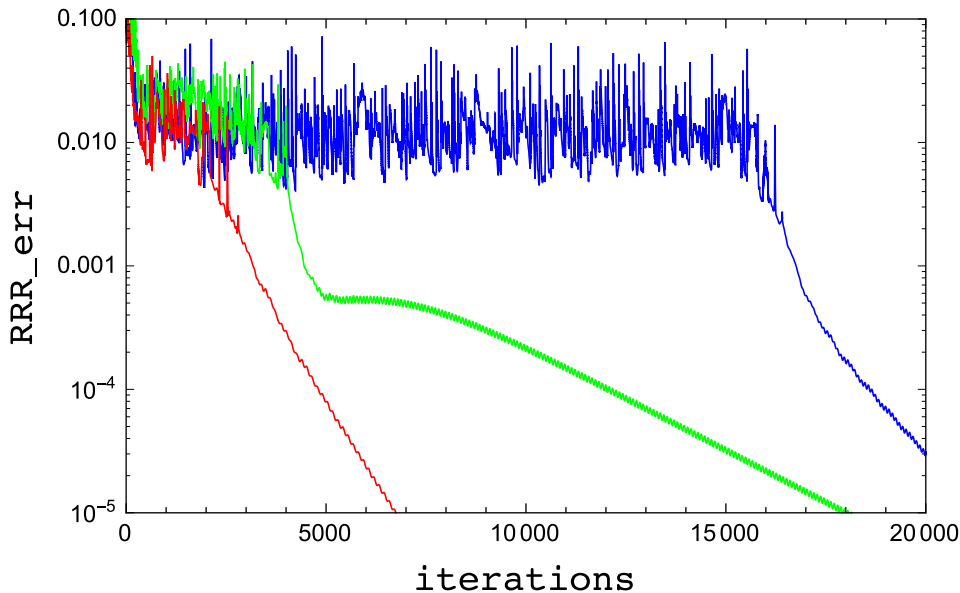}
\end{center}
\caption{Time series of \vtt{RRR_err}, the distance between constraint sets $A$ and $B$, for three runs of the algorithm on the $m=6$ LEDM instance. As \vtt{RRR_err} is analogous to loss, we see that RRR has no trouble negotiating many local minima in the course of finding solutions.}
\label{fig:fig4}
\end{figure}

\subsubsection{Synthetic letter montages}

In this application we demonstrate the online mode of NMF with the RRR algorithm, where data is processed in batches. As a technique for ``learning the parts" of images, NMF has been surpassed by more sophisticated machine learning methods. However, in order to test the RRR algorithm on a large data set, we constructed a set of 2000 images by hand where the relatively restrictive definition of ``parts" implicitly assumed by NMF applies. A sample of 16 such $40\times 40$ pixel images is shown in the left panel of Figure \ref{fig:fig5}. Each image is a montage of letters selected at random and with  two types of font: \text{\sffamily  w, x, y, z} (plain) or \text{\sffamily\itshape  w, x, y, z} (slant). Since the letters are always placed in one of four positions in the image, and there are four kinds of letters aside from the font variation, it should be possible to learn an approximate rank-16 factorization of these images. The two-layer network for this task has 16 code nodes and $40^2$ data nodes.

\begin{figure}[t!]
\begin{center}
\includegraphics[width=5.in]{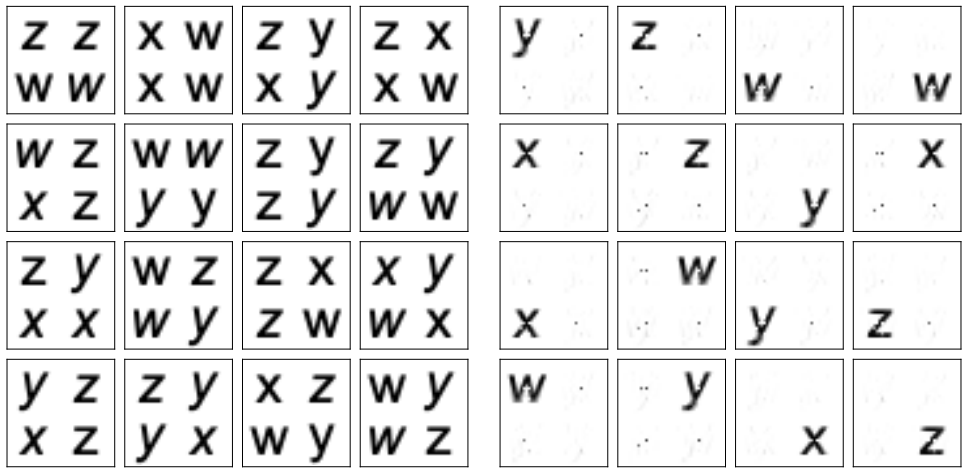}
\end{center}
\caption{\textit{Left:} Sixteen sample images from the letter-montage data set. \textit{Right:} The 16 features obtained in one run of RRR. Letters in the data occur with two fonts (plain and slant); in a successful NMF all the recovered features/letters are weakly slanted.}
\label{fig:fig5}
\end{figure}

\begin{figure}[h!]
\begin{center}
\includegraphics[width=4.5in]{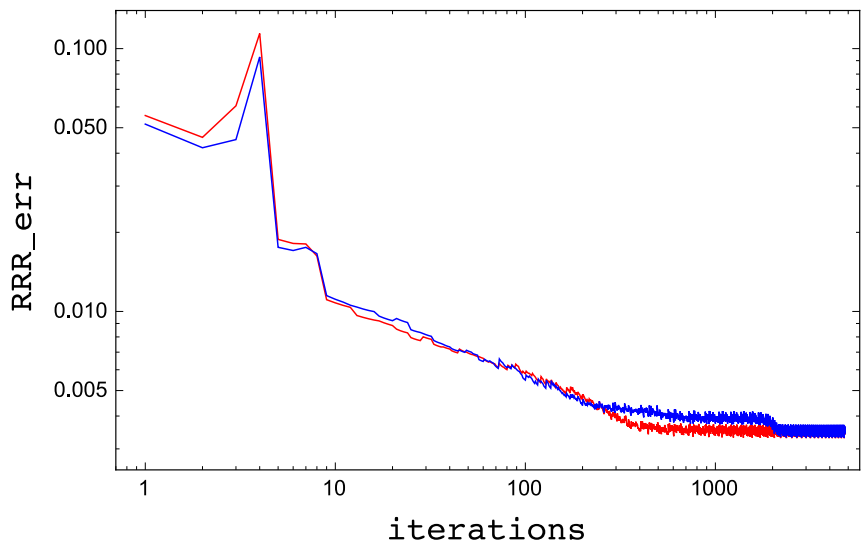}
\end{center}
\caption{Time series of \vtt{RRR_err} in two runs of online NMF for the letter-montage data set. RRR iterations are applied in blocks of 4 to batches of size 50, and there are 40 batches in the data set. In each run of 30 epochs there are altogether $4\times 40\times 30$ iterations.}
\label{fig:fig6}
\end{figure}

We processed the data in 40 batches of size 50 for 30 epochs and set a rather small limit of only 4 RRR iterations per batch. The results were not very sensitive to hyperparameter settings; we chose $\beta=1$ and $\Omega=2$. Representative time series of the RRR flow speed (\verb+RRR_err+) are shown in Figure \ref{fig:fig6} and by their near monotonic behavior indicate that these instances of NMF are easier than the LEDMs. 
However, as seen in one of the runs in Figure \ref{fig:fig6}, sometimes many iterations are needed before the algorithm manages to find the last detail that gives the minimum error. Figure \ref{fig:fig7} compares histograms of the final reconstruction error in 100 runs with the CD algorithm. Factorizations such as the one in the right panel of Figure \ref{fig:fig5} have $\verb+recon_err+\approx 0.089$. In the less successful reconstructions, for both RRR and CD, usually just one of the 16 letter/position combinations is missing and replaced by a duplicate of one of the other 15, but in the contrasting font. As in the LEDM instances the CD algorithm gave better results than MU.

\begin{figure}[t!]
\begin{center}
\includegraphics[width=6.5in]{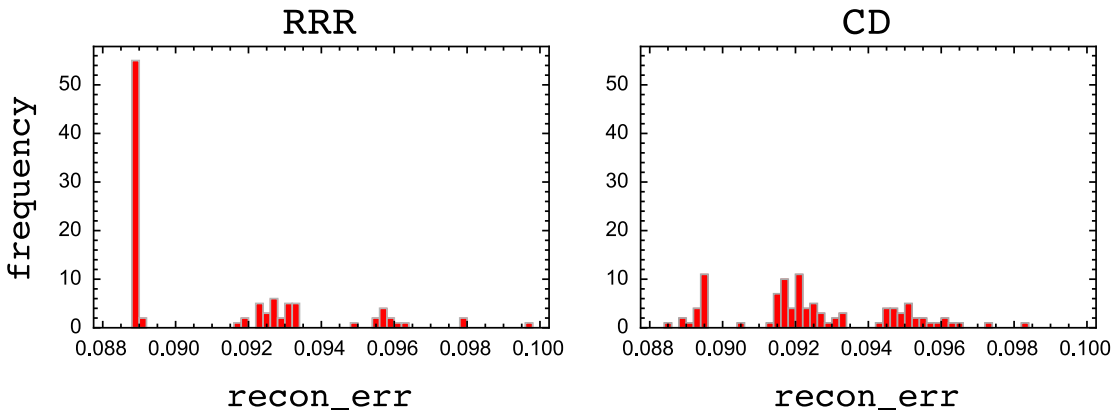}
\end{center}
\caption{Distribution of final reconstruction errors in 100 runs of RRR and coordinate descent (CD) on the letter-montage data set.}
\label{fig:fig7}
\end{figure}

\newpage 

\section{Classification}\label{sec:class}

Our loss-free or constraint-based method of training classifiers shares many elements with the method we used for NMF. The variables and constraints are defined without reference to layers apart from a layer of data nodes that receive input and a layer of class nodes at the output end of the network. As in NMF we have a weight variable $w[k,i\to j]$ for each data item $k\in K$ in the batch and edge $i\to j\in E$ of the network. One of the constraints will force the weights to reach a consensus and be independent of $k$, but now we do not additionally impose non-negativity on the consensus value. Also as in NMF there is a variable $x[k,i\to j]$ assigned to each data item and network edge. These carry the post-activation values in the network. Consensus applies to the values on edges $i\to j$ with the same origin $i$; there is also no non-negativity side constraint. Unlike NMF, the data constraint is imposed in the input layer, on the consensus values of the $x$ variables there.

A new feature in the classifier is a pre-activation node variable $y[k,j]$ for each data item $k\in K$ and node $j\in H\cup C$ (hidden and class nodes). These participate in two kinds of constraint. When $j\in H$, there is an activation-function constraint between $y[k,j]$ and the consensus value of $x$ on the same node, $x[k,j]$. This constraint also involves a bias variable $b[k,j]$, which (unlike the $x$ and $y$ variables) is constrained to be independent of $k$. We can think of the bias variables as providing a limited degree of node-specific customization to the activation functions. With this perspective it will seem less strange that the weight and bias parameters appear on different sides of the $A$/$B$ constraint splitting.

On nodes $j\in C$ the activation-function constraint on $y[k,j]$ is replaced by a class-encoding constraint. This can take many forms, analogous to the options one has for defining a loss function on the values of the class nodes. We chose the class encoding where the $(y-b)$'s on all the incorrect class nodes are constrained to be negative while $y-b$ on the correct class node is required to be greater than a positive margin parameter $\Delta$.

Recall that for NMF we motivated the use of a norm constraint, on the weights incident to each code node, in that this removed a source of solution non-uniqueness while also indirectly providing control over the relative weighting of the projection distances of the $w$ and $x$ variables. Much of this continues to be relevant for the classifier, especially when we use particular activation functions. The ReLU function has no intrinsic scale, and the step activation in Figure \ref{fig:fig2} cares only about the size of the gap it expects in its input values. In the ReLU case, a uniform rescaling of all weights in a layered network simply rescales the values at the class nodes and is equivalent to a rescaling of the margin parameter $\Delta$ for defining class boundaries. For the step activation a uniform weight rescaling is equivalent to a rescaling of the gap.

As in NMF we take advantage of the rescaling freedom of the weights to exercise control over the  metric \eqref{metric} that determines the projections. However, in order to maintain a fixed scale between the pre- and post-activation neuron values, the bilinear constraint \eqref{bilinear} will be replaced in multilayer networks by
\begin{equation*}
x\cdot w = \Omega\, y,
\end{equation*}
where $\Omega=\|w\|$ is now a constraint on the weights on the input-side of the neuron (for NMF we constrained weights on the output-side).

\subsection{Constraints}\label{sec:classcon}

Below is a summary of all the constraints in a classifier, as discussed above, partitioned into sets $A$ and $B$. The data vector for item $k$ is denoted $d[k,i]$ and the corresponding class node is $c[k]\in C$.
\begin{subequations}\label{classconstraints}
\begin{align}
\mbox{$A$ constraints}\qquad\qquad\qquad&\nonumber\\
\forall\,k\in K, i\to j\in E&:\quad x[k,i\to j]=x_A[k,i]\label{classA1}\\
\forall\,k\in K, i\in D&:\quad x_A[k,i]=d[k,i]\label{classA2}\\[10pt]
\forall\,k\in K, i\in H&:\quad x_A[k,i]=f(y[k,i]-b[k,i])\label{classA3}\\[10pt]
\forall\,k\in K, i=c[k]\in C&:\quad y[k,i]-b[k,i]\ge \Delta\label{classA4}\\
\forall\,k\in K, i\in C\setminus c[k]&:\quad y[k,i]-b[k,i]\le 0\label{classA5}\\[10pt]
\forall\,k\in K, i\to j\in E&:\quad w[k,i\to j]=w_A[i\to j]\label{classA6}\\
\forall\,j\in H\cup C&:\quad \sum_i w_A^2[i\to j]=\Omega^2\label{classA7}\\
\mbox{$B$ constraints}\qquad\qquad\qquad&\nonumber\\
\forall\,k\in K,j\in H\cup C&:\quad \sum_i x[k,i\to j]\,w[k,i\to j]=\Omega\,y[k,j]\label{classB1}\\
\forall\,k\in K,i\in H\cup C&:\quad b[k,i]=b_B[i].\label{classB2}
\end{align}
\end{subequations}

\subsection{Projections}

Many of the constraints in \eqref{classconstraints} appeared in NMF and the same projections apply. 
On the other hand, the participation of unlike variable types in some of the constraints is even more pronounced than it was in NMF. In the latter we had two types, the factors $x$ and $w$ of the factorization, but at least they came in pairs, on every edge of the network. By contrast, in our classifier some variables ($x$ and $w$) are associated with edges while others ($y$ and $b$) are associated with nodes. As we will see, the distance that defines projections must be chosen with care under these circumstances.

We will use the following distance for the variables in our classifier:
\begin{equation}\label{classDist}
\begin{aligned}
d^2(x,w,y,b)&=\sum_{{k\in K}\atop{i\to j\in E}}\left((x'[k,i\to j]-x[k,i\to j])^2+(w'[k,i\to j]-w[k,i\to j])^2\right)\\
&+\sum_{{k\in K}\atop{i\in H\cup C}}g^2(i)\left((y'[k, i]-y[k,i])^2+(b'[k, i]-b[k,i])^2\right).
\end{aligned}
\end{equation}
For simplicity we do not weight the $y$'s and $b$'s differently, and focus on the potentially more significant role of the factor $g^2(i)$ that controls the relative weight of node and edge variables. To motivate our choice for this factor we consider the activation function constraint \eqref{classA3}.

Constraint \eqref{classA3} is the side constraint that applies to the consensus values $x_A[k,i]$ when $i\in H$. Since this constraint is local in $k$, we suppress this identifier in the following. By lemma \ref{compoundconstraint}, when computing the projection $(x_A[i],y[i],b[i])\to (x'_A[i],y'[i],b'[i])$ to the activation-function side constraint we penalize changes in $x_A[i]$ by the cardinality of the $x$ variables of which $x_A[i]$ is the consensus value. By \eqref{classA1} this is the out-degree of node $i$. The projection therefore minimizes
\begin{equation}\label{actdist}
\mathrm{outdeg}(i)(x'_A[i]-x_A[i])^2+g^2(i)\left((y'[i]-y[i])^2+(b'[i]-b[i])^2\right)
\end{equation}
subject to
\begin{equation}\label{actconstraint}
x'_A[i]=f(y'[i]-b'[i]).
\end{equation}

Our principle for setting the strengths $g^2(i)$ on the hidden nodes is that the inputs $y-b$ of the activation functions should not be enslaved to the outputs $x$, and vice versa. During training we want inconsistencies in the network to be resolved, in equal measure, upstream and downstream of each neuron. To promote this behavior we set
\begin{equation}\label{hiddeng}
g^2(i)=\mathrm{outdeg}(i),\quad i\in H.
\end{equation}
Only with this rule can we expect training to behave similarly on networks with widely varying architectures (out-degrees). As there are no architecture-dependent features of the kind just described for the class nodes and constraints \eqref{classA4} and \eqref{classA5}, we introduce a hyperparameter for them:
\begin{equation}\label{classg}
g^2(i)=\Upsilon,\quad i\in C.
\end{equation}

For constraint $A$ the only projections not encountered in NMF are those for the class encoding, \eqref{classA4} and \eqref{classA5}, and the activation function side constraint, \eqref{classA3}. For the former we leave $y$ and $b$ unchanged if the relevant inequality is satisfied, or change $y$ and $b$ equally and oppositely to produce the equality case. In the latter, the projection depends on the form of the activation function $f$, where some forms can be calculated efficiently even without a look-up table. The ReLU function, its locus being the union of two half-lines, is such a case and the function we use in most of our experiments. By our choice, an effectively isotropic distance applies to both of the projections just discussed. 

In the $B$ constraints there is a slight modification of the bilinear constraint \eqref{classB1}, by the pre-activation $y$'s that was not present in NMF. The changes to the projection computation are minor and given in appendix \ref{sec:projappendix}. Note that only this constraint, for $j\in C$, has a projection that depends on the $\Upsilon$ hyperparameter.

As in NMF, the operation count for projecting to either the $A$ or $B$ constraints, dominated by \eqref{classA1}, \eqref{classA6} and \eqref{classB1}, scales as $|K| |E|$.

\subsection{Interventions for compromised data}\label{sec:compromised}

A natural objection to the use of hard constraints in real-world applications is the inevitability of compromised data. The labels on otherwise good data vectors might be wrong, or the data vectors themselves might be so severely corrupted their value for training is questionable. In this section we address this concern with simple replacements of constraints \eqref{classA4} and \eqref{classA5} and the corresponding projections. The hyperparameter associated with these replacements is a positive integer EE called the \textit{eccentric exemption}. This is a bound on the number of data in the training batch that may be exempted from the constraints. When data quality is good and the network has sufficient capacity for the classification at hand, it may turn out that fewer than EE data (or none) are exempted by the training algorithm.

\subsubsection{Corrupted vectors and possibly wrong labels}\label{sec:corruptdata}

Highly non-representative data vectors, say images of digits handwritten by only a very small fraction of the population, are by nature poor models for generalization. Data vectors of good quality but bearing wrong labels also bring no class information, and the data set would be improved by eliminating them. Both cases can be dealt with by allowing the training algorithm to ignore up to EE items of data. For example, if $\mbox{EE}=20$, then the training algorithm needs to satisfy the constraints $A$ and $B$ on only a subset $\widetilde{K}\subset K$ of cardinality $|\widetilde{K}|=|K|-20$. To implement this relaxation it suffices to eliminate EE elements of $K$ only in the class constraints, \eqref{classA4} and \eqref{classA5}, since all the other constraints associated with these data are automatically satisfied by a feed-forward pass of the (possibly corrupted) data vector through the network using the consensus weights and biases determined by the non-exempted data.

Summarizing, for this case of compromised data we replace \eqref{classA4} and \eqref{classA5} by
\begin{subequations}\label{baddata}
\begin{align}
\forall\,k\in \widetilde{K}\subset K, \;|\widetilde{K}|=|K|-\mbox{EE}&\nonumber\\
 i=c[k]\in C&:\quad y[k,i]-b[k,i]\ge \Delta\label{classA4bad}\\
i\in C\setminus c[k]&:\quad y[k,i]-b[k,i]\le 0.\label{classA5bad}
\end{align}
\end{subequations}
Projecting to this constraint requires more work, because the projection must discover the distance-minimizing subset $\widetilde{K}$. However, the additional work is modest and easy to implement. One starts by projecting, provisionally, to  the class constraints \eqref{classA4} and \eqref{classA5} for all $k\in K$ and records the net projection distance of the corresponding variables ($y$'s and $b$'s on the class nodes) for each $k$. The distance-minimizing subset $\widetilde{K}$ is obtained by keeping only those $k$, having cardinality $|K|-\mbox{EE}$, whose projection distances are smallest. The actual projection is applied only to the variables with these $k$, while those exempted are left unchanged.

\subsubsection{Wrong labels only}\label{sec:wronglabels}

We face a situation intermediate to the two considered so far when the data vectors are good and only some of the labels are wrong. In this case the class constraints \eqref{classA4} and \eqref{classA5} should be applied to all $k\in K$, but with the modification that for an exempted subset of data the class node might be different from what is specified by the data's label. We again let EE denote the number of exempted data and replace the class constraints by
\begin{subequations}\label{badlabel}
\begin{align}
\forall\,k\in K&\nonumber\\
& i=c^*[k]\in C:&& y[k,i]-b[k,i]\ge \Delta\\
&i\in C\setminus c^*[k]:&& y[k,i]-b[k,i]\le 0\\[10pt]
\forall\,k\in\widetilde{K}&\subset K, \;|\widetilde{K}|=|K|-\mbox{EE}:&& c^*[k]=c[k].
\end{align}
\end{subequations}
Here $c^*[k]\in C$ is the label selected in training; only the non-exempted subset $\widetilde{K}$ is required to match the label $c[k]$ of the data.

In addition to safeguarding generalization against wrong labels, by relaxing the class constraint in this way we have a method that in principle can fix bad labels. Classifiers that possess this feature, or \textit{relabeling classifiers}, are key to a new type of generative model described in section \ref{sec:replearn}.

Projecting to constraint \eqref{badlabel} is similar to the projection to \eqref{baddata}. The quantity used to determine the subset $\widetilde{K}$ is the excess projection distance, that is, the net projection distance of the $y$'s and $b$'s on the class nodes when $c^*[k]$ is forced to equal $c[k]$ rather than be allowed to be any of the class nodes. If up to EE data in the training batch have positive excess distance, then the training algorithm may ignore all their labels and use instead the label that gives the smaller distance. The distance minimizing projection, in general, sorts the excess projection distances and exempts those EE data which have the largest excess distances.

\subsection{Training}

As in NMF, training a classifier starts with initialization (at the outset and between batches), running iterations of RRR for the projections described above, and monitoring suitable metrics to assess progress. We use the same initialization strategy as in NMF, that minimizes randomness by exactly satisfying many easy-to-satisfy constraints. The weight variables are initialized to satisfy the consensus and norm constraints, \eqref{classA6} and \eqref{classA7}, where the consensus values are sampled from a uniform distribution and then normalized. All initial bias variables are set at zero (and therefore satisfy their consensus constraint). Initializing the $x$'s and $y$'s is done exactly as in the usual forward pass. For each data item the data vector is used as the consensus value $x_A$'s in the data layer. The consensus weights/biases and activation function are then used to propagate $y$'s and consensus $x_A$'s through the network, and all $x$'s are set to their consensus values. Upon completion of this initialization for all items in the training batch, all constraints are satisfied except the class constraints \eqref{classA4} and \eqref{classA5}, or their alternatives, \eqref{baddata} or \eqref{badlabel}, when we wish to accommodate compromised data.

When starting another training batch we use the same initialization just described except that for the consensus weights and biases we use the final consensus weights ($w_A$) and biases ($b_B$) of the previous batch.

The $A$-$B$ constraint discrepancy, \vtt{RRR_err}, is the same as the distance expression \eqref{classDist} with primed/unprimed quantities replaced by their $A$ and $B$ projection counterparts and averaged over items in the training batch (analogous to how it was defined in \eqref{RRRerr} for NMF). We note that after only few iterations it may happen that the final \vtt{RRR_err} may exceed its initial value. This is because initially all of the discrepancy is concentrated in the class constraints \eqref{classA4} and \eqref{classA5} and may increase, in aggregate, when allowed to redistribute over all the constraints.

As in NMF we control the amount of work by specifying a cutoff \verb+RRR_iter+ in the number of RRR iterations per batch and a value \verb+tol+ for \verb+RRR_err+ below which further iterations are deemed unnecessary.
 
There are three classification errors of interest. All are defined in the usual way, as the fraction of wrong classifications over particular data sets. Classifications are computed by a feed-forward of the data vector using the consensus weights/biases (at test time) and deciding class by the class node having the maximum value of $y-b$ (only one of which will exceed the margin $\Delta$ when all constraints are satisfied). All three classification errors are reported after every epoch of training. Two of them, \verb+train_err+ and \verb+test_err+, are computed for the respective data sets in their entirety at the end of the epoch, while \verb+batch_err+ is computed upon completion of each training batch and averaged over batches in the epoch. As is common practice in gradient based methods, the order of the training data is randomly shuffled before each epoch of training.

\subsection{Experiments}

For the experiments in this section we used the programs\footnotemark[\value{footnote}] \texttt{RRRclass.c} for simple classification and \verb+RRRclass_x.c+ when some number EE of (eccentric) data are exempted because of poor quality data vectors. For comparisons, in the non-exempted case, we used \texttt{scikit-learn}'s \texttt{MLPClassifier} function. To keep the models being trained the same, only layered, fully-connected networks without weight regularization were used, and the activation function was always ReLU. For the types of data studied, the simple SGD optimizer in \texttt{MLPClassifier} outperformed the others. That left only the batch size and initial learning rate ($\eta_\mathrm{init}$) as SGD hyperparameters we had to set.

\footnotetext{\texttt{github.com/veitelser/LWL}}

\subsubsection{Synthetic Boolean data}\label{sec:majgate}

\begin{figure}[t!]
\begin{center}
\includegraphics[width=3.5in]{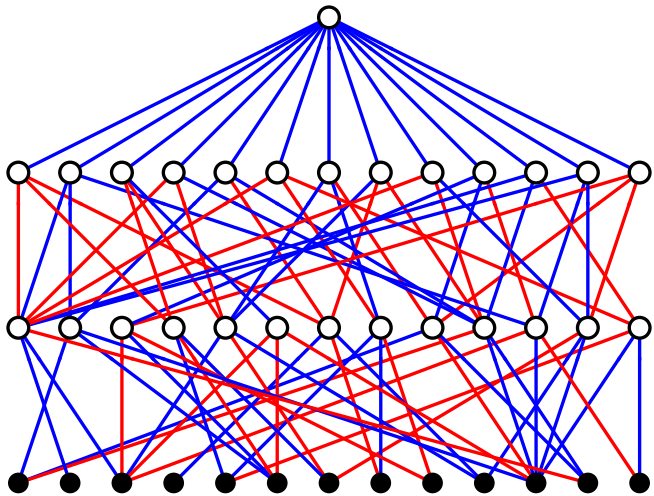}
\end{center}
\caption{One of the majority gate circuits (depth 3) used to generate data for the first classification experiment. The truth value of the circuit is computed by majority gates (white nodes) at the top node and two hidden layers of 13 nodes from 13 Boolean inputs (black nodes). Edge colors correspond to the absence/presence of a \textsc{Not} before input to the majority gate.}
\label{fig:fig8}
\end{figure}

Real-world classification conflates two stages of generalization that we might want to study independently. The first stage aims to learn what kinds of data vectors the network should expect to see, while the second stage attempts to impart structure to those ``typical" data vectors, structure that is consonant with the data labels. A Boolean function on $m$ arguments is a nice way of generating synthetic data that eliminates the first stage of generalization, because all $2^m$ Boolean vectors are valid data. We use \textsc{Not} and odd-input majority gates (rather than \textsc{And} and \textsc{Or}) to build our circuits, as this automatically ensures the two classes --- defined by the truth value after the final gate --- have equal cardinality. The difficulty of generalization is controlled by the depth $n$ of the circuit. Figure \ref{fig:fig8} shows a circuit on $m=13$ variables of depth $n=3$. In all our circuits the hidden layer majority gates take input from three randomly selected nodes in the layer below and \textsc{Not} gates are assigned randomly to edges.

In our experiments we fixed $m=13$ and randomly partitioned the full set of Boolean inputs into training and test sets of equal size, 4096. While the depth $n=2$ data was relatively easy for both RRR and SGD, the higher depth data was a challenge to learn. We used ReLU activation and based the network architecture on an identity for simulating majority gates with ReLU. Consider a majority gate receiving an odd number $p$ of inputs with negations corresponding to the $-1$ elements of the weight vector $w=(w_1,\ldots, w_p)\in\{1,-1\}^{\times p}$. With a bias set as
\begin{equation}
b(w)=({\textstyle \sum_{i=1}^p w_i}-1)/2,
\end{equation}
then
\begin{equation}\label{majsim}
x'=\mathrm{relu}\left(w\cdot x-b(w)\right)-\mathrm{relu}\left(w\cdot x-b(w)-1\right)
\end{equation}
simulates the majority gate when \textsc{F/T} are mapped, respectively, to $0/1$ in the inputs $x$ and output $x'$. To be able to simulate the logic of the hidden layers it therefore suffices to have twice the number of ReLUs as majority gates. To perform the class-defining function of the final (top) majority gate we do not need another ReLU there, since by sending oppositely signed signals to the two output nodes the class may be correctly encoded. Thus with architecture $13\to 26\to 26\to 2$ a ReLU network can in principle exactly represent the Boolean function in Figure \ref{fig:fig8}.

Since the ReLU function is scale invariant, by appropriately rescaling the biases $b(w)$ and $b(w)+1$, the conclusion about the network architecture will continue to hold when the weights satisfy our normalization $\|w\|=\Omega$. Recalling that the inputs to our ReLUs is  $y=w\cdot x/\Omega$, we see that the output $x'$ in \eqref{majsim} will be diminished, after rescaling, by $\sqrt{1/p}$ relative to the inputs $x$. The net effect of these rescalings in a network with $h=n-1$ hidden layers of $2m$ ReLU neurons, each with exactly three non-zero, equal magnitude input weights, is multiplicative. In particular, when the number of majority gates with value 1 in the penultimate layer of the circuit changes from $(m-1)/2$ to $(m+1)/2$, thereby changing the truth-value/class, the corresponding $y$ values at the two output nodes of the ReLU network change by
\begin{equation}\label{deltabound}
\frac{1}{\sqrt{2m}}\left(\frac{1}{\sqrt{3}}\right)^h.
\end{equation}
By setting the margin parameter $\Delta$ below this value we know that it is possible to exactly represent the corresponding Boolean function. However, because we do not impose (in training) the constraint that the hidden layer ReLUs have exactly three non-zero and equal magnitude input weights, we cannot rule out that RRR is able to succeed with a $\Delta$ greater than this bound. In fact, this is what we find.

Having to properly set the margin $\Delta$ may seem to put constraint-based classifiers at a disadvantage relative to loss-based (e.g. cross-entropy) classifiers. However, our RRR experiments show that results are not all that sensitive to this parameter, and a good setting can be found with few trials. It should also not be overlooked that most loss functions have parameters as well, such as the temperature in the cross-entropy function\footnote{When the weights in the last layer are unconstrained and not subject to regularization, then their scale subsumes the role of the temperature.}.

Since the space of hyperparameters and training protocols is large, our first experiment is focussed on the best use of resources. Specifically, we are interested in minimizing the work as measured by \verb+GWMs+ to achieve a given classification accuracy. Is it better, when using RRR, to do many iterations per batch, say consistently achieving $\verb+batch_err+=0$, and few epochs, or the other way around? As a representative case we used the depth $n=3$ data, batch size $|K|=128$, and the ``fast" time-step $\beta=1$. For the three hyperparameters that control the projections we chose $\Omega=2$, $\Upsilon=1$, and $\Delta=0.1$. Of these only $\Delta$ has a significant effect and we present those details later. In the meantime, we note that \eqref{deltabound} for $h=2$ gives the bound $\Delta<0.065$, so that by setting $\Delta=0.1$ RRR is being challenged to find a somewhat stronger class separation than promised by the majority gate simulation analysis.

Figure \ref{fig:fig9} compares \verb+train_err+ when 50, 100, or 200 RRR iterations are performed per batch, with the number of epochs decreasing as 2000, 1000, and 500 to keep the work constant. We see that $\verb+RRR_iter+=100$ is the most efficient, at least for batch size 128, and we fix this in the subsequent experiments. With this setting RRR has no trouble getting to $\verb+train_err+=0$, even though about 30 epochs are needed before $\verb+batch_err+=0$ is achieved (not shown). 

\begin{figure}[t!]
\begin{center}
\includegraphics[width=5.5in]{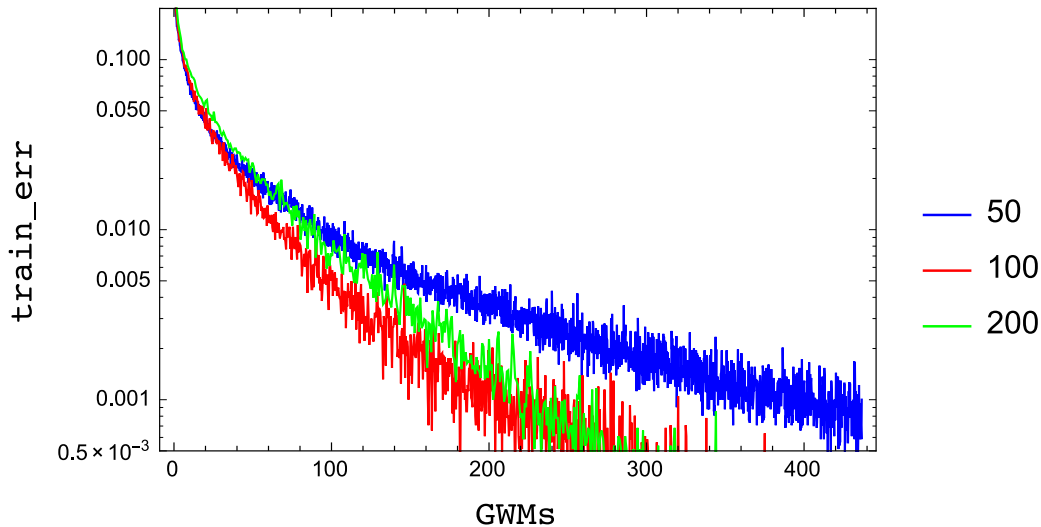}
\end{center}
\caption{Training error vs. work for the depth-3 majority-gate circuit data, compared for RRR with 50, 100 and 200 iterations per batch. Results are averages of 10 runs with random initial weights.}
\label{fig:fig9}
\end{figure}

Figure \ref{fig:fig10} shows how well the RRR trained networks generalize for the protocol just described, now for four values of the margin $\Delta$. The scatter of points combines the results of 10 runs and the 10 final epochs of training. We see that the training data can still be represented accurately when $\Delta$ is increased from 0.05 to 0.1, and that this improves generalization. However, further increases in $\Delta$ compromised generalization by an amount similar to the increase in training error.

\begin{figure}[t!]
\begin{center}
\includegraphics[width=5.5in]{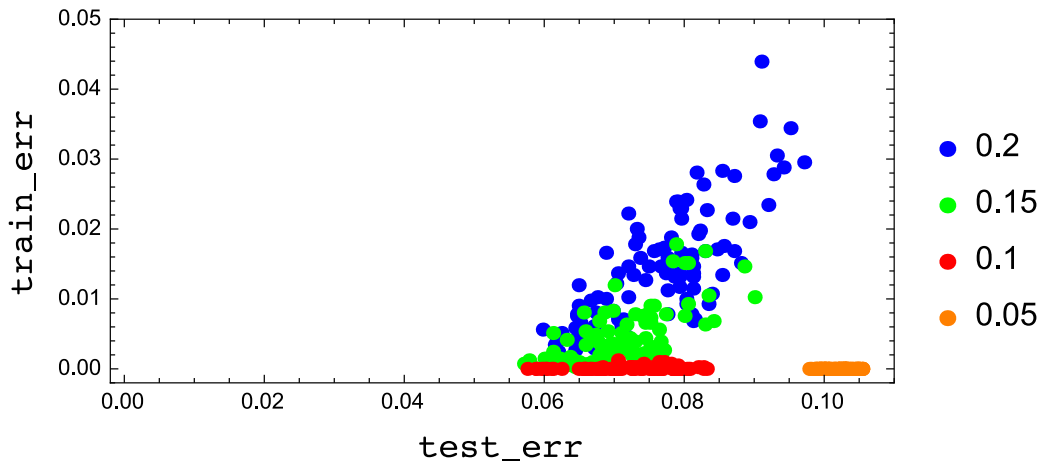}
\end{center}
\caption{Distribution of training and test (generalization) error for RRR-trained networks on the depth 3 majority-gate circuit data. The four distributions differ only in the value of the margin $\Delta$ used in training.}
\label{fig:fig10}
\end{figure}

In the final experiment with majority-gate data we compare generalization (\verb+test_err+) for data generated by circuits of increasing depth. We fixed all the hyperparameters at the same setting determined above, except that we used a further doubled $\Delta=0.2$ for the $n=4$ data (where again RRR had no trouble reaching $\verb+train_err+=0$). As before, we used the fully connected architecture with $n-1$ hidden layers of 26 neurons.
These results, shown in Figure \ref{fig:fig11}, are compared in Figure \ref{fig:fig12} with those obtained by the \texttt{scikit-learn} classifier trained on the same data, architecture, activation function and batchsize. The best results for the gradient based method were obtained with the simple SGD optimizer in the adaptive learning rate mode, and training was terminated when the loss improved by less than $10^{-5}$, typically after about 1000 epochs. To get good SGD results the initial learning rate had to be large, $\eta=0.5$ for $n=2,3$ and $\eta=1.0$ for $n=4$. Even so, in about 8\% of the trials on the $n>2$ data the final training error was above 1\%.

Although the data in this classification task are small by current standards, the ability of both RRR and SGD to generalize, even with modest precision, from seemingly random strings of bits is truly `superhuman'. Whereas memorizing 4096 items is well within the scope of human savants, gleaning an underlying pattern and applying it to another 4096 items has never, to the best of our knowledge, been demonstrated by a human subject.

\begin{figure}[t!]
\begin{center}
\includegraphics[width=5.in]{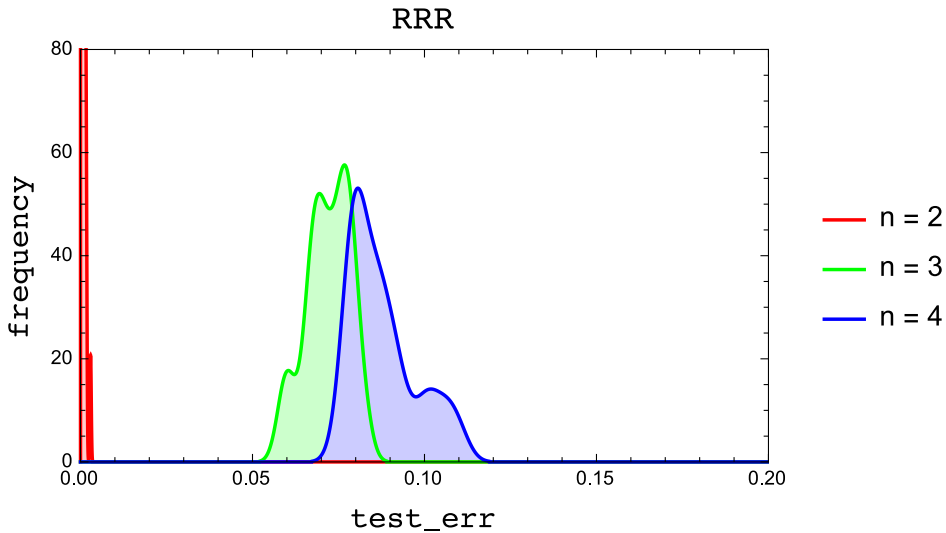}
\end{center}
\caption{Behavior of RRR generalization (distribution of test error in final 10 epochs of 10 trials) with increasing depth $n$ of the majority-gate generated data.}
\label{fig:fig11}
\end{figure}

\begin{figure}[t!]
\begin{center}
\includegraphics[width=5.in]{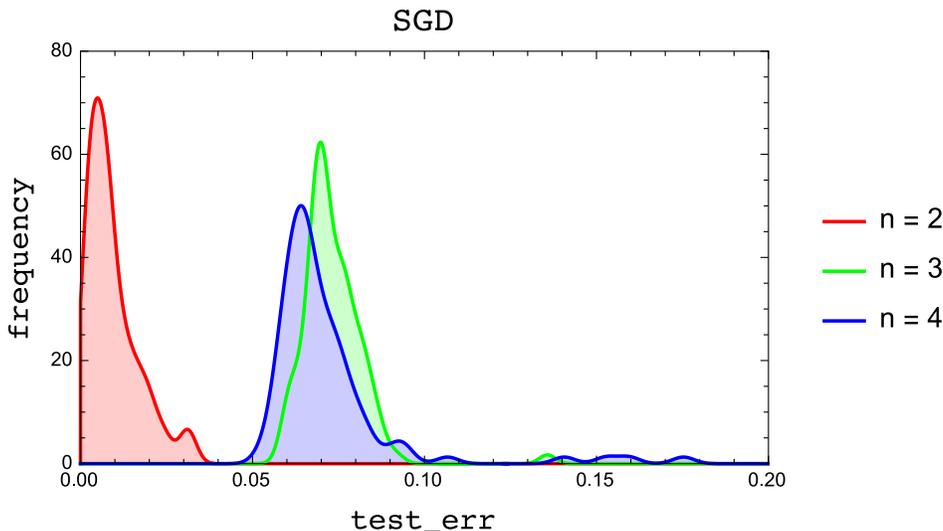}
\end{center}
\caption{Same as Figure \ref{fig:fig11} but for 100 trials of SGD. Aside from some outliers, SGD does better on average for the deepest data but, unlike RRR, fails to get perfect generalization for $n=2$ data on the small architecture.}
\label{fig:fig12}
\end{figure}

\subsubsection{MNIST with eccentric exemptions}\label{sec:EE}

We used the MNIST data set to test the strategy of improving generalization by exempting a given number of items during training. As described in detail in section \ref{sec:corruptdata}, this is where we slightly attenuate the $A$ constraints by dropping the class constraint on \verb+EE+ items, the ``eccentric exemptions", where the exempted items are determined dynamically by the projection principle that the distance to the class constraint, of the retained items, is minimized.

Even with a fixed architecture and choice of activation we have at our disposal another means of potentially improving generalization: the margin $\Delta$ we impose on the correct-class node. Naively, increasing $\Delta$ should improve generalization because it increases the separation of classes in the output layer. But this is a very different mechanism than letting the network learn to exempt data that look eccentric within the rest of the data. Since it is hard to theorize which strategy will be better for generalization, we approach this as an experiment. All results were obtained using \verb+RRRclass_x.c+\footnotemark[\value{footnote}] and $10^4$ training and test items from the MNIST data set.

\footnotetext{\texttt{github.com/veitelser/LWL}}

As in the preceding study our hyperparameter optimization (aside from \verb+EE+ and $\Delta$) was only rough and yielded $\beta=1$, $\Omega=2$, and $\Upsilon=1$. To minimize the work we selected a small, fully connected architecture with only one hidden layer: $784\to 20\to10$. Anticipating that only a few percent of the data would call for exemption, we trained on relatively large batches so that all batches would have some exempted items. We chose $|K|=1000$ and let $\verb+ee+=\verb+EE+/|K|$ vary between 0 and 1.5\%. For this batch size $ \verb+RRR_iter+=100$ gave the most efficient training, as quantified by \verb+GWMs+, and we trained for 200 epochs.

\begin{figure}[t!]
\begin{center}
\includegraphics[width=4.5in]{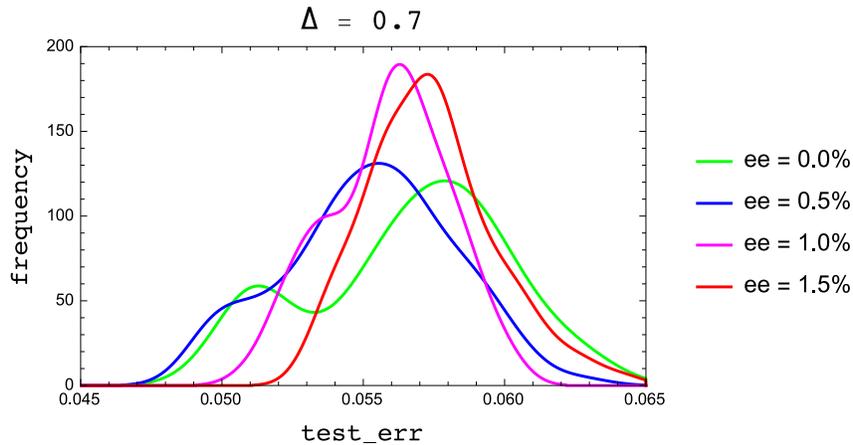}
\end{center}
\caption{Behavior of RRR generalization (distribution of test error in final 10 epochs of 10 trials) when the number of exempted MNIST training data is varied and the margin parameter is fixed at $\Delta=0.7$.}
\label{fig:fig13}
\end{figure}

\begin{figure}[t!]
\begin{center}
\includegraphics[width=4.5in]{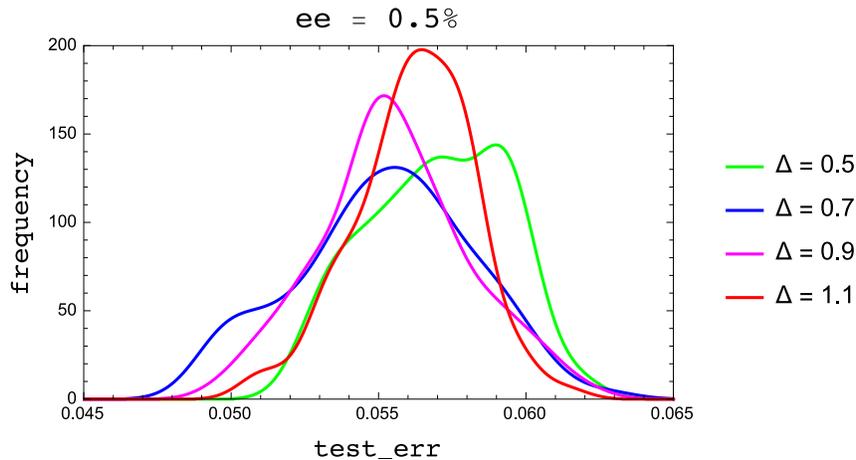}
\end{center}
\caption{Same as Figure \ref{fig:fig13} but with a varying $\Delta$ and fixing the exempted data at 0.5\%.}
\label{fig:fig14}
\end{figure}

Figures \ref{fig:fig13} and \ref{fig:fig14} show the variation in the distribution of test error when respectively \verb+ee+ and $\Delta$ are tuned around their best values ($\verb+ee+=0.5\%$, $\Delta=0.7$). Each distribution sampled the final 10 epochs of 10 trials. We see that the effect of both parameters on generalization is small, at least when compared against the widths of the distributions. Although few in number, the 50 exempted items over the entire training data, shown in Figure \ref{fig:fig15} for one run, are in many cases strikingly bad examples of handwriting. SGD does poorly on the same data, batch size and small architecture after optimizing the initial learning rate. The test error in 100 trails had mean 0.073 and was never below 0.061.

With hindsight, the MNIST data set was not the best choice to demonstrate the value of exempting items in training to improve generalization. Only the final few percentage points pose any difficulty, that is, accuracy on the same small fraction of outliers that the exemption strategy ignores!

\begin{figure}[t!]
\begin{center}
\includegraphics[width=4.5in]{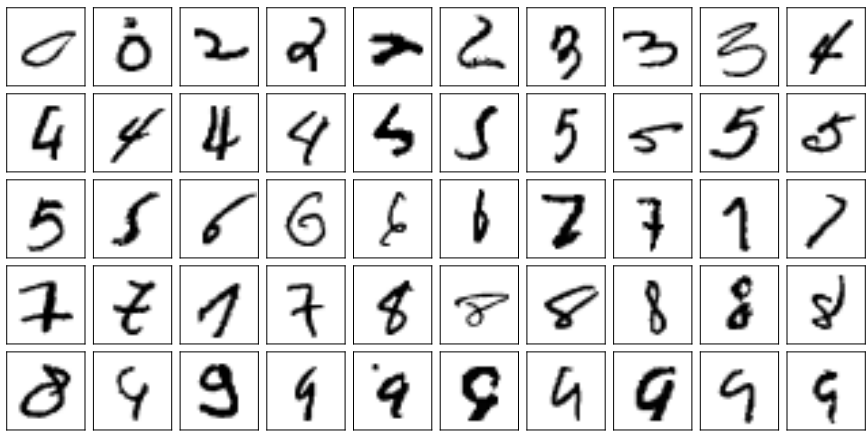}
\end{center}
\caption{Final exempted MNIST training data in one run.}
\label{fig:fig15}
\end{figure}

\section{Representation learning}\label{sec:replearn}

In signal processing language, representation learning is a scheme for lossy compression of data vectors, $x\to z$, where the code vector $z$ retains all the salient information necessary to reconstruct a good approximation of $x$ given only $z$. Generative models, an application of representation learning, impose the additional demand of knowing how to construct codes $z$, \textit{de novo}, that decode to vectors $x$ that are hard to distinguish from actual data.

Autoencoder networks, such as shown in Figure \ref{fig:fig3}(c), are widely used in representation learning. The lower half of the network, the encoder, receives data $x$ in the input layer and computes the corresponding code $z=\mathcal{E}(x)$ in the code layer. In a well trained autoencoder, the decoding performed by the upper half of the network should closely match the original data, or $\mathcal{D}(z)\approx x$. In addition to a loss associated with the quality of the reconstruction, standard (gradient-based) training methods introduce additional loss functions to give the code vectors useful attributes.

 We now describe a scheme for constructing generative models that avoids loss functions. Loss functions are attractive, in part, because they enable gradient methods to act on probability distributions, when these are parameterized and impose structure on the distribution of code vectors. Since our training method does not rely on gradients, we can work without parameterized probability distributions. In fact, in our scheme we never need to go beyond empirical distributions, that is, representative collections of data and code vectors. Our model consists of two parts: (i) an autoencoder with additional constraints on the code vectors, and (ii) a special case of the relabeling classifier described in section \ref{sec:wronglabels}. The adjectives ``variational" and ``adversarial" do not apply to either of these parts.
 
 Using constraints we will try to train our autoencoder to construct codes $Z$ that have the following three properties:
 \begin{itemize}

\item $Z$ should be a \textit{invertible}. This means that for every $x\in \mathcal{D}(Z)$ there is essentially a unique code $z\in Z$ that decodes to $x$ (sufficiently distinct codes never decode to the same $x$). This unique code is of course $z=\mathcal{E}(x)$, so $\mathcal{D}$ has an inverse and it is $\mathcal{E}$. We will impose this property on $Z$ with the constraint $\mathcal{E}(\mathcal{D}(z))\approx z$ on samples $z\in Z$.
 
 \item $Z$ should be \textit{data-enveloping}, in the sense that if we encode any data $x$, $\mathcal{E}(x)=z$, then we always can find a $z'\in Z$ such that $z'\approx z$.
 
 \item Finally, we choose $Z$ to be \textit{disentangled}, that is, $Z=Z_1\times\cdots\times Z_{|C|}$ is a product of empirical distributions for each node in the code layer.
 
 \end{itemize}
 We make the distinction between the \textit{code} being disentangled (our usage) and the \textit{representation}, provided by the encoder $\mathcal{E}$, being disentangled (common usage). The former is the weaker property but the one that is more straightforward to realize with constraint-based training. The second component of our generative model, the relabeling classifier, restores the full functionality. This component, the relabeling classifier, does not rely on the code $Z$ being disentangled. We choose $Z$ to be disentangled mostly for the technical benefit of easy sampling. Invertible, data-enveloping, and disentangled codes will be referred to as iDE codes.
 
It seems unrealistic to us, to expect that encoders can always bijectively map data into a simple code space, such as is implicit in the widely used multi-variate normal distributions for codes. Although the universal approximator property (of deep network encoders) is powerful, a connected domain for $Z$ should come at a price when the data $X$ has disconnected components. In that scenario the decoder must either have strong discontinuities or introduce interpolation artifacts. In more concrete terms, a seamless morphing of an MNIST 3 into an 8 (along a path in code space) is less a display of semantic brilliance than a casualty of code space crowding. In our autoencoder, by contrast, the encoding of actual data is only required to be injective, that is, $\mathcal{E}(X)\subset Z$.

To complete the generative model we train a classifier $\mathcal{C}$ that can recognize special codes $z\in Z$ that are hard to distinguish from the encodings by $\mathcal{E}$ of true data. Rejection sampling of $Z$, with this classifier serving as filter, will then generate $z$ that decode to $\mathcal{D}(z)=x$ that are hard to distinguish from true data.

In more detail, the binary classifier $\mathcal{C}$ is constructed as follows. After training the autoencoder we generate a body of labeled codes as the union $\tilde{Z}=\tilde{Z}(0)\cup \tilde{Z}(1)$, where $\tilde{Z}(1)=\mathcal{E}(X)$ is the encoded data with label 1 (genuine) and $\tilde{Z}(0)$ is a uniform sampling of the iDE code $Z$ with size of our choosing and label 0 (fake). Since the genuine codes, by the enveloping property, occupy some fraction of the code space $Z$, when we train $\mathcal{C}$ to distinguish genuine from fake codes we must respect the fact that we cannot expect to get the false-positive rate to be zero. We do not know this false-positive rate and train $\mathcal{C}$ with a \textit{false-positive allowance}, or \verb+fpa+, as a hyperparameter to be optimized. The true-negative rate, on the other hand, should be zero because we know the codes $\tilde{Z}(1)$ are genuine. If we set a small value for \verb+fpa+, then $|\tilde{Z}(0)|$ should be chosen to be large so that the number of false-positive data seen by the classifier, $\verb+FPA+=\verb+fpa+ \times |\tilde{Z}(0)|$, is reasonable.

\begin{figure}[t]
\begin{center}
\includegraphics[width=5.5in]{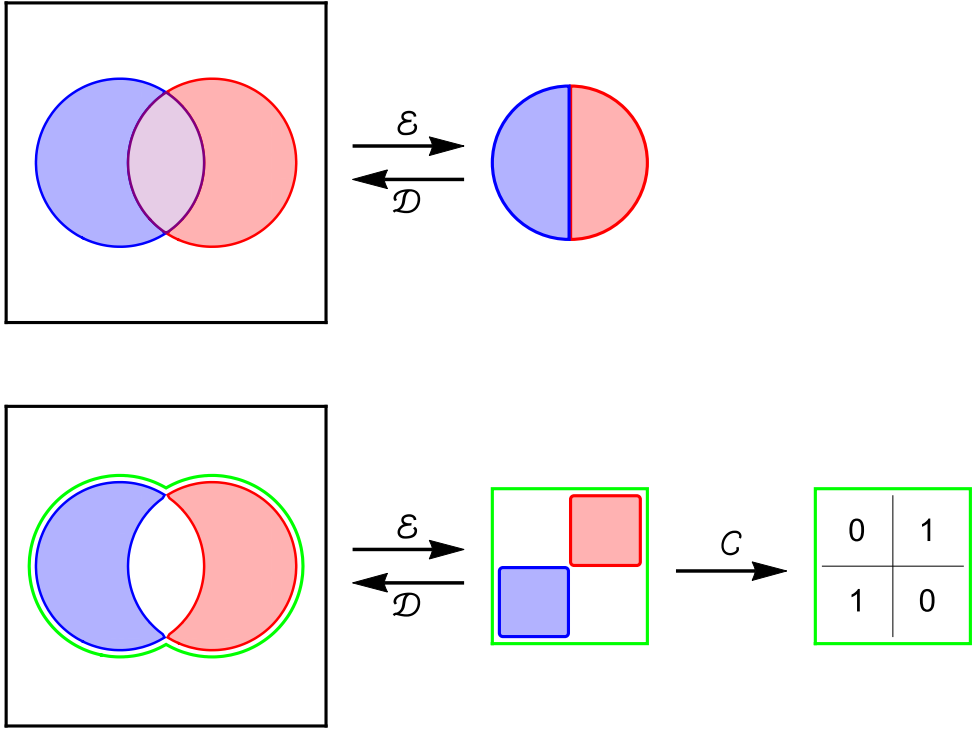}
\end{center}
\caption{Comparison of a conventional autoencoder (top row) and the proposed model based on iDE codes (bottom row), when the data is not a connected set --- rendered as the \textsc{Xor} of two disks. In the conventional design, decoding will have interpolation artifacts (purple region of ``Venn diagram"). This is avoided by iDE encoding, which only seeks to envelop the data (green boundary curve) while also constraining the representation to be disentangled (tensor product rendered as a green square). The generative model for iDE codes requires, in addition, a classifier $\mathcal{C}$ trained to identify codes that correspond to data. When the data codes occupy half of the code space, as in this cartoon, the best setting of the \vtt{fpa} training parameter would be $1/2$.}
\label{fig:fig16}
\end{figure}

Summarizing, the generative model comprises the trio $(\mathcal{C},\mathcal{D},Z)$, where the last two are products of the autoencoder. To generate fake data that is hard to distinguish from genuine data we take samples $z\in Z$, accept those classified as genuine (true) by $\mathcal{C}$, and output $\mathcal{D}(z)=x$. If we trained the classifier $\mathcal{C}$ with a false-positive allowance of \verb+fpa+, then the rate of accepted code samples or fake data will be \verb+fpa+. Figure \ref{fig:fig16} contrasts this design with more conventional designs, such as variational autoencoders.

The key hyperparameters that control the difficulty of constructing a generative model of the kind just described are the number of nodes $|C|$ in the code layer of the autoencoder and the false-positive allowance, \verb+fpa+, when training the classifier. To take advantage of an enlarged code space free of interpolation artifacts, $|C|$ should be larger than what is usually considered optimal. On the other hand, when $|C|$ is large it may be too easy for the classifier to distinguish genuine and fake codes. The resulting small \verb+fpa+ would require an unreasonably large body of training data $|\tilde{Z}(0)|$ to see examples of viable fakes.

\subsection{Autoencoder details}\label{sec:autoencoder}

The variables and constraints that apply to our autoencoder, where codes are constrained to have the three iDE properties, are not that different from those of the basic classifier of section \ref{sec:classcon}. The differences, in brief, are the following:
\begin{itemize}
\item The network is cyclic, where output nodes are not distinct nodes but identified with the input/data nodes $D$.
\item The data constraint \eqref{classA2} takes two forms: one for data vectors at the data nodes $D$ and another for code vectors at the code nodes $C$.
\item An activation constraint \eqref{classA3} is imposed at all nodes on which there is no data/code constraint. 
\item Constraints \eqref{classA4} and \eqref{classA5} coming from the class label are absent.
\end{itemize}
The cyclic structure of the network, combined with the data/code constraints, imposes the reconstruction property $\mathcal{D}(\mathcal{E}(x))\approx x$ for the data as well as the invertibility property $\mathcal{E}(\mathcal{D}(z))\approx z$ of iDE codes. To see how the other two properties of iDE codes are imposed we need to describe how the autoencoder constructs the code $Z$.

The construction of $Z$ takes place in the setting of data batches. A data batch is the union $K=K_d\cup K_c$, where $K_d$ is the same as the data batch in the simple classifier, where data item $k\in K_d$ has data vector $d[k]$ (but now there is no class label). The code batch $K_c$ is a collection of code vectors, and the enveloping/disentangled properties derive from how the codes $c[k]$, $k\in K_c$, are constructed.

The codes $c[k]$ are $|K_c|$ uniform samples from the product $Z=Z_1\times\cdots\times Z_{|C|}$ of empirical distributions at the code nodes that the training algorithm manages. This ensures the disentangled property. In order for $Z$ to be enveloping as well, the training algorithm constructs each $Z_i$ from the encodings of a suitably large body of data. Since these 1D distributions are not complex, they are well represented by relatively few samples. A practical solution, with all 1D distributions of size $|K_d|$, is to set
\begin{equation}
\forall\, i\in C: \quad Z_i=\{\mathcal{E}(d[k])_i\;\colon k\in K_d\},
\end{equation}
where the network parameters of the encoder $\mathcal{E}$ are those obtained from training on the previous batch. After initial transients, due to randomly initialized weights and biases, these 1D distributions quickly settle down from one batch to the next provided $|K_d|$ is not too small.

We get a system of constraints that needs to make the fewest exceptions for node type (data, code, hidden) when the domain of the data $d[k,i]$ and code $c[k,i]$ values is the same as the image of the activation function $f$. When the former are the discrete set $\{0,1\}$, we use the step activation function shown in Figure \ref{fig:fig2}. For image data we scale pixel values into the interval $[0,1]$ and instead use the continuous ``zigmoid" function shown in Figure \ref{fig:fig17}.

\begin{figure}[t]
\begin{center}
\includegraphics[width=3.in]{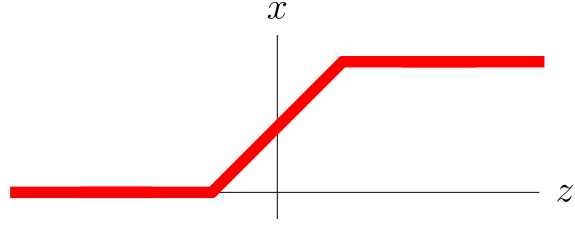}
\end{center}
\caption{The zigmoid activation function. The non-constant part interpolates between 0 and 1 over the range $\Delta$.}
\label{fig:fig17}
\end{figure}

Below is the complete set of constraints used for training the autoencoder. Note that the symbol for post-activation node variables is always $x$, that is, we drop the earlier use of $z$ for nodes in the code layer.
\begin{subequations}
\begin{align}
\mbox{$A$ constraints}\qquad\qquad\qquad\qquad&\nonumber\\
\forall\,k\in K_d\cup K_c,\, i\to j\in E&:\quad x[k,i\to j]=x_A[k,i]\label{autoA1}\\[10pt]
\forall\,k\in K_d, i\in D&:\quad x_A[k,i]=d[k,i]\label{autoA2}\\
\forall\,k\in K_c, i\in C&:\quad x_A[k,i]=c[k,i]\label{autoA3}\\[10pt]
\forall\,k\in K_d\cup K_c,\,i\in D\cup C\cup H&:\quad x_A[k,i]=f(y[k,i]-b[k,i])\label{autoA4}\\[10pt]
\forall\,k\in K_d\cup K_c,\, i\to j\in E&:\quad w[k,i\to j]=w_A[i\to j]\label{autoA5}\\
\forall\,j\in D\cup C\cup H&:\quad \sum_i w_A^2[i\to j]=\Omega^2\label{autoA6}\\
\mbox{$B$ constraints}\qquad\qquad\qquad\qquad&\nonumber\\
\forall\,k\in K_d\cup K_c,\,j\in D\cup C\cup H&:\quad \sum_i x[k,i\to j]\,w[k,i\to j]=\Omega\,y[k,j]\label{autoB1}\\[10pt]
\forall\,k\in K_d, i\notin D&:\quad b[k,i]=b_B[i]\label{autoB2}\\
\forall\,k\in K_c, i\notin C&:\quad b[k,i]=b_B[i].\label{autoB3}
\end{align}
\end{subequations}
As the kinds of constraints are no different from those in the simple classifier, the same projections apply. The only difference, owing to the absence of an output layer with class constraints, is that there is a single metric parameter for the $y$'s and $b$'s
\begin{equation}
g^2(i)=\mathrm{outdeg}(i),\quad i\in D\cup C\cup H,
\end{equation}
and no $\Upsilon$ hyperparameter. Related to this are exceptions to the projections to the consensus side constraints \eqref{autoA4} when $i\in D$ or $i\in C$. Since the $x$ for these nodes is fixed by the data or code vector, only $y$'s and $b$'s are changed. These are easy projections, where the $y$ and $b$ of each data/code node is shifted by the same, oppositely signed amount to make the argument of the activation function produce the intended value.

\subsubsection{Training}

Training on a batch, comprising both data vectors and code vectors, begins with a ``feed-around" the cyclic network starting from the data/code layer $x$'s. Using the weights and biases from the random initialization, or parameters from the previous batch of training, this sets all the other $x$'s and the $y$'s as well (we need to be careful not to overwrite the $x$'s in the data/code layer). In addition to initializing variables for RRR optimization, this is also how we define the two reconstruction errors:
\begin{align}
(\vtt{data_err})^2 &=\frac{1}{|K_d| |D|}\sum_{{k\in K_d}\atop{i\in D}}\left(x_A[k,i]-f(y[k,i]-b_B[i])\right)^2\\
(\vtt{code_err})^2 &=\frac{1}{|K_c| |C|}\sum_{{k\in K_c}\atop{i\in C}}\left(x_A[k,i]-f(y[k,i]-b_B[i])\right)^2.
\end{align}
A combination of these is also the RRR constraint discrepancy at the start of iterations, since only constraint \eqref{autoA4} (for the data and code nodes) is not automatically satisfied by our initialization. More generally,
\begin{equation}
\begin{aligned}
(\vtt{RRR_err})^2=\frac{1}{|K_d|+|K_c|}\;\sum_{k\in K_d\cup K_c}&\left(\sum_{i\to j\in E}\left(\Delta x[k,i\to j]^2+\Delta w[k,i\to j]^2\right)\right.\\
+&\left.\sum_{i\in D\cup C\cup H}g^2(i)\left(\Delta y[k,i]^2+\Delta b[k, i]^2\right)\right),
\end{aligned}
\end{equation}
where $\Delta x= x_A-x_B$, etc. Since both the data vectors and the code vectors have elements in the unit interval, the rms-errors \vtt{data_err} and \vtt{code_err} should be small compared to 1 in a good representation. In the case of \vtt{RRR_err} we are mostly interested in how it decreases from one epoch to the next, not its value in absolute terms.

The network architecture, in particular the number of nodes in the code layer $|C|$, is one the autoencoder's most important hyperparameters. Another hyperparameter, the size of the code batch $|K_c|$, will depend on what we choose for $|C|$. Additional hyperparameters, pertaining to the constraints and the RRR algorithm, are the number of iterations per batch, \vtt{RRR_iter}, the step size $\beta$, the norm $\Omega$ on the weights, and the margin parameter $\Delta$ that sets the input range over which the activation function $f$ (step or zigmoid) changes from 0 to 1.

\subsection{Relabeling classifier details}

The relabeling classifier we use in the generative model is a special case of the one described in section \ref{sec:wronglabels}. There are two class nodes: $c_1$ for (encodings of) true data and $c_0$ for fakes. Labeled data for classification comprises true codes, $K(1)$ (encodings of actual data vectors), and samples of the disentangled code (fakes), $K(0)$. When $k\in K(1)$, the classifier should produce a negative value for $y[k,c]-b[c]$ for $c=c_0$ and a value that exceeds $\Delta$ for $c=c_1$. Most items $k\in K(0)$ should give the opposite result. However, as explained above, we relax this constraint so that items $k\in\widetilde{K}(0)\subset K(0)$, where $|\widetilde{K}(0)|=\vtt{FPA}$ is the false positive allowance, are allowed to satisfy the true data constraints instead.

Projecting to the modified class constraint is similar to the projections for compromised data described in section \ref{sec:compromised}. For each item in $K(0)$ we compute two projection distances, $d(0)$ (given label) and $d(1)$ (relabeled). Of the differences $d^2(0)-d^2(1)$ that are positive (for which the relabeling would be closer), we perform a sort and apply up to \vtt{FPA} relabels of those items at the top.

This relabeling classifier has all the hyperparameters of the simple classifier with the addition of the false positive allowance rate \vtt{fpa}. The size of the fake code batch $|K(0)|$ should be viewed as a hyperparameter and set large enough that relabeled codes are well sampled, that is, so $\vtt{FPA}=\vtt{fpa}\times |K(0)|$ is large. Epoch to epoch progress in training is monitored by the true-negative rate \vtt{tn} and false-positive rate \vtt{fp}. Training is successful when \vtt{tn} is small and \vtt{fp} does not exceed \vtt{fpa} by too much.

\subsection{Experiments}

For the training experiments in this section we used the programs\footnotemark[\value{footnote}] \vtt{RRRauto.c} for the autoencoder and \vtt{RRRclass_fp.c} for true/fake code classification with a false positive allowance. We do not present direct comparisons because the software in \texttt{scikit-learn}, the source of our mainstream algorithms, does not have the required functionality. Comparisons with state-of-the-art representation learning algorithms are planned for the future, after the RRR software has received some enhancements, e.g. convolutional layers.

\footnotetext{\texttt{github.com/veitelser/LWL}}

\subsubsection{Binary encoding}\label{sec:binaryencoding}

An interesting toy example of representation learning was considered by \cite{rumelhart1985learning}, in the same article that introduced the back-propagation algorithm. The question is whether a network can be trained to encode all $2^n$ 1-hot vectors of length $2^n$ into binary codes of length $n$, and then to decode these back to the original 1-hot vectors. We will see that a two-layer autoencoder network ($2^n\to n\to 2^n$), with step activation (Fig. \ref{fig:fig2}) at all nodes, is capable of this task. For the $n=3$ network with sigmoid activation, \cite{rumelhart1985learning} found that the SGD algorithm also was able to find parameters that solved this autoencoder problem. However, the sigmoid function allowed, and usually included, the number $1/2$ in addition to 0 and 1 in the code. We are not aware of any follow-up studies, such as the behavior of training as $n$ grows.

By using the 2-valued step activation function we were able to train networks that solved the strict binary encoding problem, that is, for codes $\{0,1\}^{\times n}$. The training data comprised $2^n$ 1-hot vectors $d[k], k\in K_d$, with constraints $\mathcal{D}(\mathcal{E}(d))=d$ applied at the data layer, and the $2^n$ binary codes $c[k], k\in K_c$, with constraints $\mathcal{E}(\mathcal{D}(c))=c$ applied at the code layer. By the nature of the problem, an autoencoder (with step activation) that has been successfully trained just on $K_d$ would also be able to autoencode the codes $K_c$, though the converse is not true. We found that training on $K_d$ and $K_c$ jointly worked better than training on just $K_d$.

Because of the discontinuous activation ``function" (Fig. \ref{fig:fig2}), we were not surprised that a somewhat large $\Omega$ parameter was favored by this application. We found that $\beta=0.5$ and $\Omega=100$ gave good results up to $n=5$, the largest instance we tried. In appendix \ref{sec:binarycodeappendix} we show that the normalization constraint on the weights places constraints on the margins of the step-activations in a perfect encoder/decoder. Specifically, the margins $\Delta_c$ and $\Delta_d$ at respectively the code and data nodes must satisfy
\begin{equation}
\Delta_c\le \frac{2}{\sqrt{|D|}},\qquad \Delta_d\le \frac{1}{\sqrt{|C|}}.
\end{equation}
In the equality case the corresponding weights are unique (up to the $(2^n)!$ ways of mapping the $2^n$ 1-hot vectors to the integers $0,1,\ldots, 2^n-1$) and differ only in sign. The uniform value\footnotemark[\value{footnote}] $\Delta=0.4$ is consistent with both of these except for $n=5$, where we used $\Delta=0.34$ instead. When we set  $\Delta=0.4$ for $n=5$, RRR often finds a near solution (proximal points on the two constraint sets) and correct binary encodings when the activation function, after training, is replaced by the usual zero-margin step function. With $\Delta$ set close to its maximum value, the trained autoencoder weights (see Fig. \ref{fig:fig18}) are narrowly distributed, differing (in each layer) mostly in sign.

\footnotetext{\texttt{RRRauto} imposes the same margin on all the step activations.}

\begin{figure}[t]
\begin{center}
\includegraphics[width=5in]{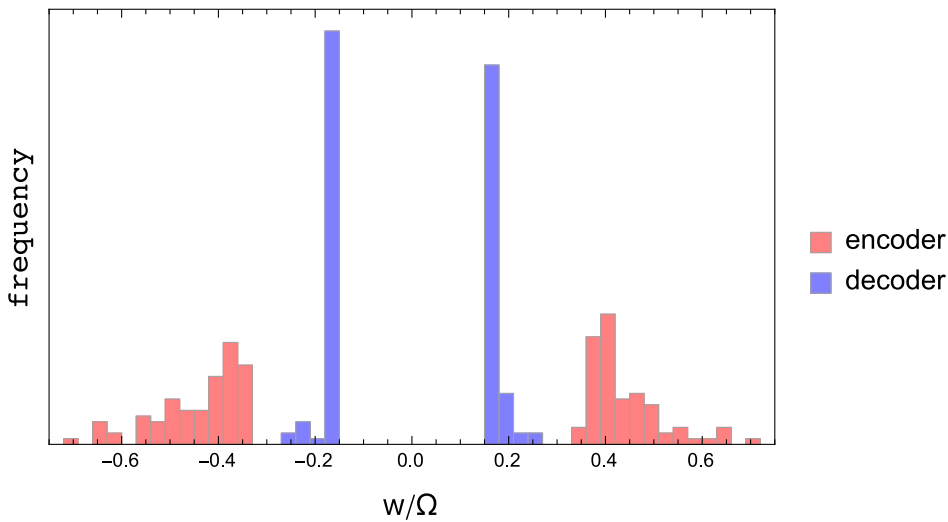}
\end{center}
\caption{Final autoencoder weights for the $n=5$ binary encoding task.}
\label{fig:fig18}
\end{figure}

The discovery of binary encoding/decoding weights, starting from random, broadly distributed weights, does not have the ``aha" behavior we observed in some of the other combinatorial tasks, such as LEDM factorization (Fig. \ref{fig:fig4}). Instead, we find that \vtt{RRR_err} behaves very similarly across training runs, making incremental progress and differing significantly only when \vtt{RRR_err} is very small, where some runs succeed while others get trapped and fail to find a perfect solution. Figure \ref{fig:fig19} shows this for $n=5$. The success rate is 80\% for $n=3$ and drops to 30\% and 25\%, respectively,  for $n=4$ and $n=5$.

\begin{figure}[t]
\begin{center}
\includegraphics[width=4.in]{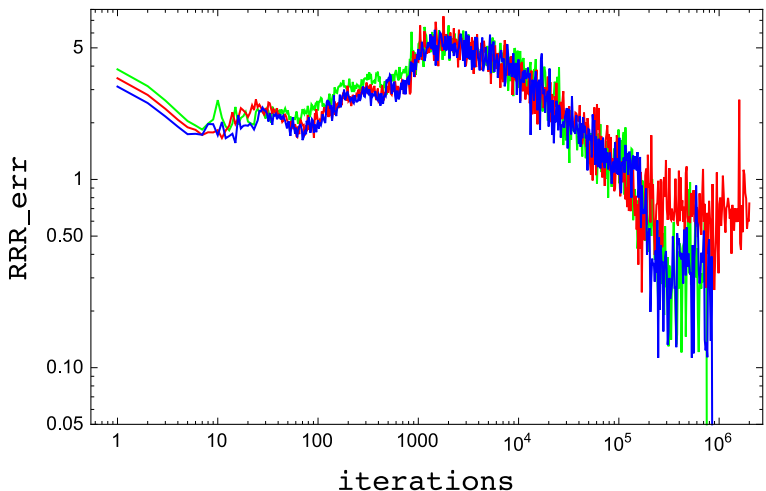}
\end{center}
\caption{Three training runs for $n=5$ binary encoding, two successful (blue, green) and one unsuccessful (red).}
\label{fig:fig19}
\end{figure}

\subsubsection{MNIST Digits}\label{sec:mnistgen}

To demonstrate generative models based on iDE codes we return to the MNIST data set. Using $10^4$ items from the training data, we trained an autoencoder with architecture ${784\to 200\to 10\to 200\to 784}$ and zigmoid activation (Fig. \ref{fig:fig17}) as described in section \ref{sec:autoencoder}. Hyperparameters were selected, by trial and error, to give the fastest reduction in \vtt{data_err} for a given amount of work (\vtt{GWMs}), since the other reconstruction error, \vtt{code_err}, was always significantly smaller. This yielded $\Omega=10$ and $\Delta=0.4$. Since the larger number and size of the hidden layers make this representation learning task less combinatorial in nature than the binary encoding problem, we used the Douglas-Rachford time step $\beta=1$. Also, given the $\vtt{data_err}>\vtt{code_err}$ asymmetry, the training batches were structured to have 10 times as many data constraints, $\mathcal{D}(\mathcal{E}(x))=x$, as code constraints, $\mathcal{E}(\mathcal{D}(z))=z$. Using a combined batch size of $200+20$, we applied 50 RRR iterations per batch for 40 epochs, or just under 7000 \vtt{GWMs} of work. The final reconstruction errors were $\vtt{data_err}=0.153$ and $\vtt{code_err}=0.040$. To put the first of these in perspective, a linear model with the same number of code nodes cannot have an error less than 0.185 (by SVD analysis). The improvement over the linear model comes with the added benefits that the encoding is non-negative and disentangled. The ten distributions whose product give the final iDE code are shown in Figure \ref{fig:fig20}.

\begin{figure}[t]
\begin{center}
\includegraphics[width=5.5in]{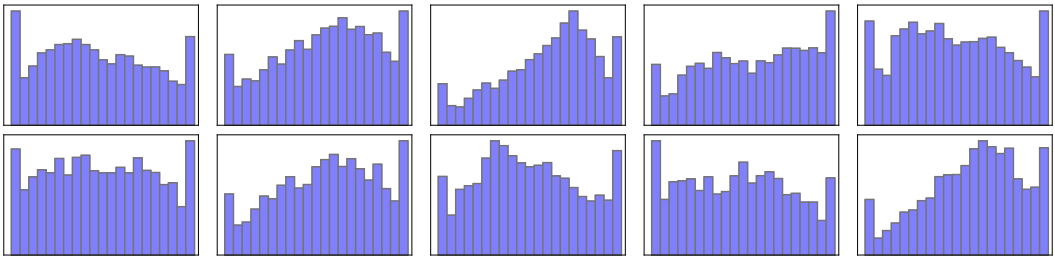}
\end{center}
\caption{iDE encoding of MNIST digits for a code layer of 10 neurons. The iDE code is the product distribution of the 10 distributions above, where the peaks at the ends of the intervals correspond to 0 and 1 activation of the respective code layer neuron.}
\label{fig:fig20}
\end{figure}

Figure \ref{fig:fig21} gives a subjective assessment of the quality of the autoencoder. The images in the top panel were generated by passing 50 items from the MNIST test data through the autoencoder. Clearly there is much room for improvement! The images in the bottom panel were generated by applying the decoder $\mathcal{D}$ to 50 samples of the iDE code shown in Figure \ref{fig:fig20}. Recall that our generative model is based on the principle that some fraction of the lower samples in Figure \ref{fig:fig21} are of the same quality as the upper samples, and that a classifier $\mathcal{C}$ can be trained for this task.

The data for training $\mathcal{C}$ comprised the encodings (by the autoencoder's $\mathcal{E}$) of $10^4$ MNIST training data with label ``genuine," and $2\times 10^4$ samples of the iDE code with label ``fake." We did not explore the effect of changing the number of fakes in the training data. The doubled number of fakes simply makes the statement that the number of fake data is in principle unlimited in our scheme. For $\mathcal{C}$ we used ReLU activation and settled on architecture $10\to 50\to 50\to 50\to 2$ after checking that an additional layer or doubling the width did not significantly improve results. A rough trial-and-error search yielded nearly the same hyperparameters that worked well for the classifier of section \ref{sec:majgate}: $\beta=1$, $\Omega=5$, $\Upsilon=1$, and $\Delta=0.1$. Our strategy for selecting the false-positive allowance parameter, \verb+fpa+, was to start high and decrease by 5\% in subsequent runs until we noticed that the false-positive and true-negative rates (\verb+fp+ and \verb+tn+) on the training data made no progress toward their targets of \verb+fpa+ and 0\%, respectively. This yielded the setting $\verb+fpa+=25\%$.

Figures \ref{fig:fig22} and \ref{fig:fig23} show the evolution of the two error rates, averaged on batches, the entire training data, and also a test data set constructed exactly as the training data but using encodings of the MNIST test items for the ``genuine" codes. Each epoch of $3\times 10^4$ items was partitioned into batches of size 500 and 1000 RRR iterations were applied to each batch. After 20 epochs (3400 \verb+GWMs+) \verb+tn+ is still decreasing, while \verb+fp+ is holding steady at values above 25\%. We interpret this to mean that the quality of the allowed fraction of false-positives (fakes) is improving as well, because it is being defined relative to an improving representation of genuine codes. Not surprisingly, we see that fluctuations in \verb+fp+ and \verb+tn+ are anticorrelated. 

\begin{figure}[t]
\begin{center}
\includegraphics[width=4.5in]{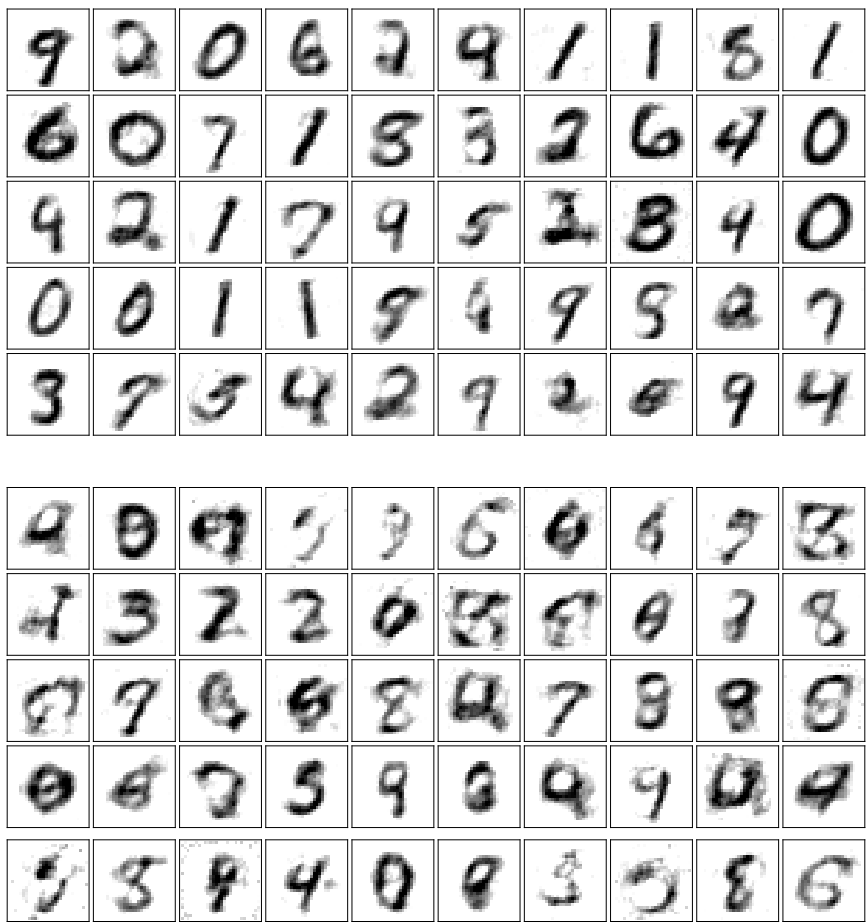}
\end{center}
\caption{\textit{Top:} The result of passing items from the MNIST test data through the autoencoder. \textit{Bottom:} Images produced by decoding samples of the iDE code shown in Figure \ref{fig:fig20}.}
\label{fig:fig21}
\end{figure}

\begin{figure}[t]
\begin{center}
\includegraphics[width=5in]{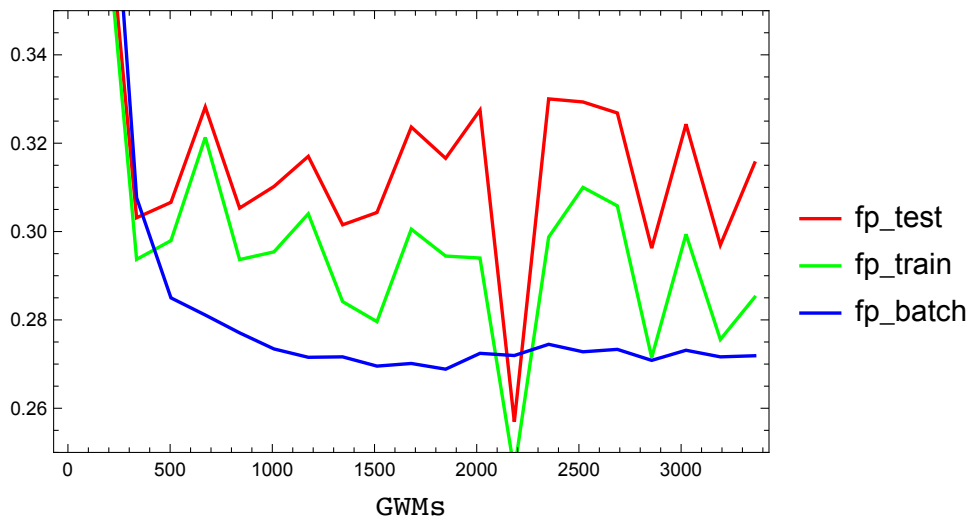}
\end{center}
\caption{Evolution of the three false-positive rates over 20 epochs when training on MNIST iDE codes with a false-positive allowance of 25\%.}
\label{fig:fig22}
\end{figure}

\begin{figure}[h]
\begin{center}
\includegraphics[width=5in]{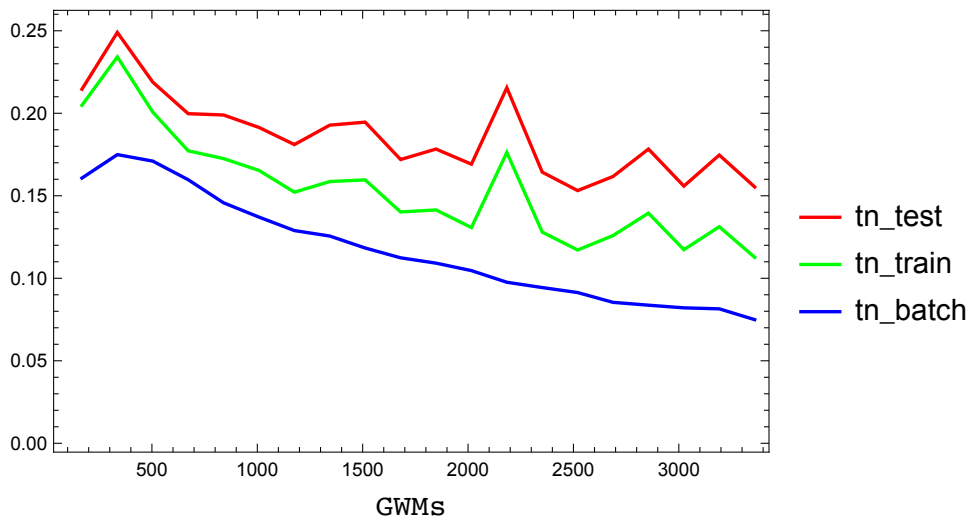}
\end{center}
\caption{Same as Figure \ref{fig:fig22} but for the true-negative rates.}
\label{fig:fig23}
\end{figure}

The final false-positive rate of the classifier, on test data, was 31.6\%. This is the rate at which randomly drawn iDE codes are accepted as genuine and, it is hoped, decode to images that resemble MNIST digits. Figure \ref{fig:fig24} shows a sample of 50 such images. One might argue that codes deemed genuine in the training set would still meet the definition of a generative model, since these too were generated \textit{de novo} from the iDE code. And because $\verb+fp+=28.5\%$ is lower for the training data, giving a more discriminating classifier, the quality of the generated images would improve. However, the output of such a generative model is limited by the number of training data unless one is willing to invest some amount of classifier-training work with each fake that the model produces.

\begin{figure}[t]
\begin{center}
\includegraphics[width=4.5in]{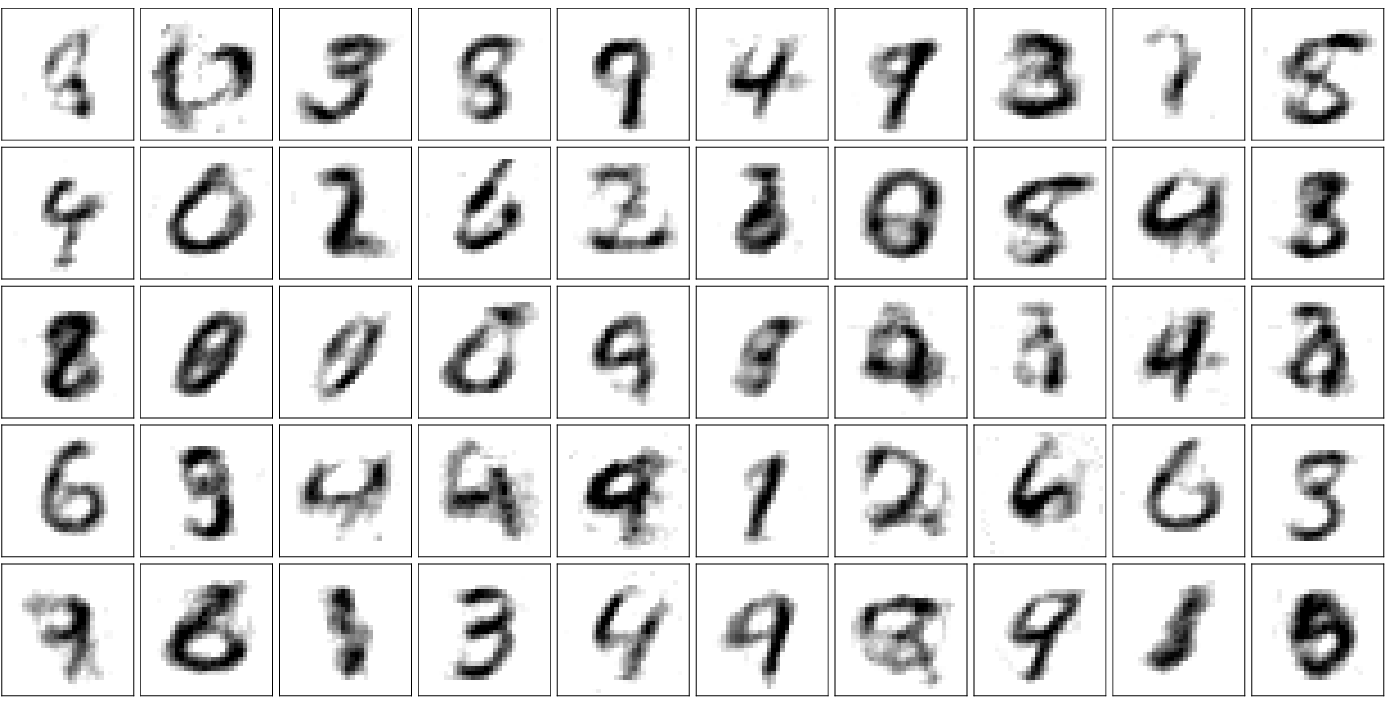}
\end{center}
\caption{The result of decoding iDE codes that were classified as ``genuine" by the classifier of the generative model.}
\label{fig:fig24}
\end{figure}

\section{Conclusions}\label{sec:conclusions}

The utility of neural networks for representing and distilling complex data cannot be overstated. Does this utility derive from the forgiving nature of the platform, on which even unsophisticated and often undisciplined training usually succeeds? Or have neural networks risen to the top because they are exceptionally well suited for gradient descent, the training algorithm one would like to use because of its intuitive appeal? One way to address these questions, and the one taken in this paper, is to try a radically different approach to training.

Our approach avoids gradients and loss functions and was inspired by phase retrieval, where the most successful algorithms take steps derived from constraint projections. We used the general purpose RRR algorithm which divides the constraints into two sets $A$ and $B$ such that in each iteration the algorithm exactly solves, in effect, one half of the training problem. The ``projection steps" of RRR still manage to be local: their computation distributes, not only over the network (at the level of individual neurons) but also over the items in the training batch.

We demonstrated the new approach in three standard settings: non-negative matrix factorization (a single-layer network with constraints), classification, and representation learning. For each we featured one application that was ``small and tricky" and another that was ``large and wild". The new approach was shown to be superior to gradient-based methods for the former and seemed to also hold promise for the latter. Even so, just as one sailor walking across the deck of an aircraft carrier will not alter its course, it is unrealistic to expect these findings to substantially impact the course of neural network training. Our concluding remarks will therefore focus on findings that might translate into the standard paradigm.

Formulating training as a constraint satisfaction problem brought up a number of questions that relate to generalization. In the constraint approach the weights $w$ and activations $x$ are treated on a more equal footing, especially in the constraint $x\cdot w=y$ that relates these to the pre-activation $y$ of the receiving neuron. Moreover, when defining the projections one has flexibility in setting the relative compliance of these variables, by breaking a rescaling invariance with the norm specification $\Omega$ on $w$. Small $\Omega$ favors resolving discrepancies at the receiving neuron while large $\Omega$ pushes changes to the inputs (to neurons lower in the network). By being forced to declare the relative compliance of weights and activations, the constraint approach drew attention to an interesting handle on generalization through the depth behavior of training. It is interesting that in all our experiments the best results were always obtained with $\Omega>1$.

Generalization also stands to gain from a new interpretation of activation functions made possible by the constraint approach. Although for comparison purposes our experiments mostly used the ReLU function, in the new approach an activation ``function" is a general constraint on pre- and post-activation pairs $(y,x)$. A gapped step-activation constraint (Fig. \ref{fig:fig2}) may improve generalization because training is forced to find weights that avoid gray areas for the neuron inputs. In the one experiment where this activation constraint was used, binary encoding in section \ref{sec:binaryencoding}, the best results were obtained when the gap/margin of the step was set near the maximum possible value (consistent with the norm constraint on the weights). When the sigmoid function was used with SGD on this problem \citep{rumelhart1985learning}, the codes found were not always strictly binary because the continuous function allowed activation $x=1/2$.  

Giving the training algorithm the latitude to exempt items from training by an objective criterion --- the projection distance --- is another way to improve generalization. An extension of this idea, giving the algorithm additionally the power to attach new labels to a bounded fraction of the training data, was used in the relabeling classifier. Within the constraint framework the implementation of these features is automatic, the bounds on the number of exempted items, or candidates for relabeling, being the only parameters. Introducing this functionality in gradient based methods, though possible, would not be nearly as direct.

The generative model based on iDE codes (section \ref{sec:replearn}) shows that working with constraints as opposed to loss functions is not an impediment to the creation of elaborate learning systems. In this application we saw that invertibility of the autoencoder and the disentangled property of the code are easily expressed through constraints. A special case of the relabeling classsifier, with a false positive allowance, then serves to identify codes in an expanded, ``enveloping" code space that decode to good fakes. This design for a generative model seems natural and can surely also be implemented in the loss function framework.

The article ``Tackling Climate Change with Machine Learning" \citep{rolnick2019tackling} is a call for engagement on probably the single most critical issue of our time. While comprehensive in surveying applications, the authors neglect to turn the mirror on themselves. The training of neural networks for natural language processing, an industry still in its infancy, is already a major consumer of energy \citep{strubell2019energy}.  It is to draw attention to this side of machine learning that we deliberately chose \textit{not} to use distributed processing on a massive scale, made possible by the constraint based approach, as a selling point. Wall clock time, number of training epochs, etc. should always take a back seat to energy consumption. While hardware developments shift the landscape, there is an algorithmic component of energy consumption unique to neural networks: the total number of weight multiplications over the course of training. This motivated the \vtt{GWMs} (giga-weight-multiplies) unit we introduced by which training algorithms can be given a fair ranking\footnote{This assumes the implementation is such that each multiply can be lumped with the associated memory access.}. We did not undertake a careful comparison with gradient methods in this regard, and it may well turn out that RRR is \textit{not} superior to SGD as measured by \vtt{GWMs}. The clear advantage RRR has in parallelizability would be vitiated by such a finding.




\newpage 
\appendix

\section{}\label{sec:RRRappendix}

This appendix is meant to be a concise, self-contained guide to the family of constraint satis\-faction algorithms to which RRR belongs. For an excellent and much more comprehensive review, see the article by \cite{lindstrom2018survey}.

\subsection{RRR as relaxed Douglas-Rachford}

The RRR algorithm derives its name from the following expression for the update of the search vector $x$,
\begin{equation}\label{RRR}
x'=(1-\beta/2)x+(\beta/2) R_B\left(R_A(x)\right),
\end{equation}
where
\[
R_A(x)=2P_A(x)-x,\qquad R_B(x)=2P_B(x)-x
\]
are reflections through the sets $A$ and $B$. The parameter $\beta$ ``relaxes" the ``reflect-reflect-average" case ($\beta=1$) that the convex optimization literature refers to as the Douglas-Rachford iteration. The projections, such as to set $A$,
\[
P_A(x)=\argmin_{x'\in A}{\|x'-x\|},
\]
define a unique point only when the constraint sets are convex. However, since non-uniqueness in the non-convex case arises only for $x$ on sets of measure zero, in our computational setting we treat the projections (as we do in software) as proper maps. Rewriting the reflections in \eqref{RRR} in terms of projections,
\begin{equation}
x'=x+\beta\left(P_B(2 P_A(x)-x) - P_A(x)\right),
\end{equation}
we see that $\beta\to 0$ corresponds to the flow interpretation. At a fixed point we have $x'=x$, and therefore $P_B(2 P_A(x)-x) = P_A(x)$ must be a solution as it lies in both $A$ and $B$. However, the fixed point itself is not in general a solution. The fixed-point/solution relationship and the attractive nature of fixed points is explained in section \ref{sec:RRRconv}.

\subsection{ADMM with indicator functions}

By using indicator functions for sets $A$ and $B$ as the two objective functions in the ADMM formalism \citep{boyd2011distributed}, the ADMM algorithm also provides a way of finding an element \eqref{AintB} in their intersection. One iteration \citep{boyd2011distributed} involves a cycle of updates on a triple of variables:
\begin{subequations}\label{ADMM}
\begin{align}
z'&=P_A(x-y)\label{ADMM1}\\
y'&=y+\alpha(z'-x)\label{ADMM2}\\
x'&=P_B(z'+y').\label{ADMM3}
\end{align}
\end{subequations}
We have followed the conventions in the ADMM review by \cite{boyd2011distributed} except in what we define to be the start and end of a cycle. Conventionally the final update is \eqref{ADMM2}, where the scaled dual variable $y$ is incremented by the difference of the two projections. For showing RRR/ADMM equivalence (see below) the projection to $B$ is the more convenient choice to end the cycle. This difference is irrelevant for fixed point behavior, where we see that $y'=y$ implies $z'=x=x'\in A\cap B$ is a solution to \eqref{AintB}. The constant $\alpha\in(0,2)$ is a relaxation parameter, where $\alpha<1$ corresponds to under-relaxation. To run ADMM the dual variables $y$ must be initialized in addition to $x$; a standard choice is $y=0$. With this initialization and $\alpha=0$, ADMM reduces to the alternating-projection algorithm. That alternating-projections often gets stuck (cycles between a pair of proximal points), when ADMM does not, shows that $\alpha\to 0$ is a singular limit.

\subsection{General properties}

The following general properties distinguish RRR and indicator-function-ADMM from other iterative algorithms.
\begin{itemize}

\item Problem instances are completely defined by a pair of projections.

\item Attractive fixed points encode solutions but, in general, are not themselves solutions.

\item The update rule respects Euclidean isometry.

\end{itemize}
The last property states that if $x_0,x_1,\ldots$ is a sequence of iterates generated by constraint sets $A$ and $B$, then for any Euclidean transformation $T$, the constraint sets $T(A)$ and $T(B)$ would generate the sequence $T(x_0),T(x_1),\ldots$ . This follows from the Euclidean norm minimizing property of projections and that the construction of new points from old is ``geometric". For example, the update rule
\begin{equation}\label{gamma}
x'=x+\beta\left(P_B((1+\gamma)P_A(x)-\gamma x)-P_A((1-\gamma)P_B(x)+\gamma x)\right)
\end{equation}
generalizes RRR (beyond $\gamma= 1$) and also respects Euclidean isometry.

\subsection{Unrelaxed ADMM/RRR equivalence}

RRR with $\beta=1$ is equivalent to indicator-function-ADMM with $\alpha=1$. To see this, define a shifted $x$ for ADMM by $\tilde{x}=x-y$, and use the update rules \eqref{ADMM} to determine $\tilde{x}'=x'-y'$. By \eqref{ADMM1} we have $z'=P_A(\tilde{x})$ and from \eqref{ADMM2} (with $\alpha=1$)
\begin{align*}
y'&=(x-\tilde{x})+(z'-x)\\
&=P_A(\tilde{x})-\tilde{x}.
\end{align*}
Finally, using \eqref{ADMM3}
\begin{align*}
\tilde{x}'&=x'-y'\\
&=P_B\left(P_A(\tilde{x})+P_A(\tilde{x})-\tilde{x}\right)-\left(P_A(\tilde{x})-\tilde{x}\right)\\
&=\tilde{x}+P_B\left(2 P_A(\tilde{x})-\tilde{x}\right)-P_A(\tilde{x}),
\end{align*}
we see that the shifted $x$ of ADMM has the same update rule as RRR with $\beta=1$.

\subsection{Local convergence of RRR}\label{sec:RRRconv}

Let $a\in A$ and $b\in B$ be mutually proximal points, and suppose $\|a-b\|$ is zero or sufficiently small that in a suitable neighborhood $U$ the sets $A$ and $B$ may be approximated as flats,
\begin{align*}
A&\approx\overline{A}+a\\
B&\approx\overline{B}+b,
\end{align*}
where $\overline{A}$ and $\overline{B}$ are linear spaces. Let $U=Z_\perp \oplus Z$ be the orthogonal decomposition of $U$, where  $Z_\perp=\overline{A}+\overline{B}$ is the span of the two spaces. Also decompose $Z_\perp$ orthogonally, as  $Z_\perp=X\oplus Y$, where $Y=\overline{A}\cap\overline{B}$. The two linear spaces now orthogonally decompose as $\overline{A}=C\oplus Y$ and $\overline{B}=D\oplus Y$, where $C$ and $D$ are linearly independent subspaces of $X$, or $C\cap D=\{0\}$. In the orthogonal decomposition $U=X\oplus Y\oplus Z$, we can write down the most general pair of proximal points as
\begin{subequations}\label{proxpoints}
\begin{align}
a&=(0,y,a_z)\\
b&=(0,y,b_z),
\end{align}
\end{subequations}
where $y\in Y$ is arbitrary and $a_z,b_z\in Z$ are fixed by the two flats.

Projections from a general point $(x,y,z)\in U$ have the following formulas,
\begin{subequations}\label{projformula}
\begin{align}
P_A(x,y,z)&=(P_C(x),y,a_z)\\
P_B(x,y,z)&=(P_D(x),y,b_z),
\end{align}
\end{subequations}
where $P_C$ and $P_D$ are the linear projections to the subspaces $C$ and $D$. The $\beta\to 0$ flow, now for the generalized RRR update \eqref{gamma}, takes the following form:
\begin{align}
\dot{x}&=\left((1+\gamma)P_D P_C-(1-\gamma)P_C P_D-\gamma (P_D+P_C)\right)(x)\label{xflow}\\
\dot{y}&=0\nonumber\\
\dot{z}&=b_z-a_z.\nonumber
\end{align}
We have fixed-point behavior only for $b_z=a_z$, when the proximal points coincide. From \eqref{proxpoints} the space of solutions, or $a=b$, is parameterized by $y\in Y$. However, for each such solution point the flow is free to choose any $z\in Z$ for its fixed point. To establish convergence to any of these fixed points we need to check that $x\to 0$ under the RRR flow. The same check applies in the infeasible case, $b_z\ne a_z$, since by \eqref{projformula} we see that $x\to 0$ ensures the projections $P_A$ and $P_B$ converge to the two proximal points \eqref{proxpoints}. To prove this result we need the following lemma:
\begin{lem}\label{lem1}
If $C \oplus C_\perp$ and $D \oplus D_\perp$ are two orthogonal decompositions of $X$, where \newline  $C+D=X$, then $C_\perp\cap D_\perp=\{0\}$.
\end{lem}
\begin{proof}
From
\begin{align*}
C_\perp&=\{x\in X\colon u^T x=0,\; \forall\, u\in C\}\\
D_\perp&=\{x\in X\colon v^T x=0,\; \forall\, v\in D\},
\end{align*}
it follows that if $x^*\in C_\perp\cap D_\perp$, then
\[
(u+v)^T x^* =0,\; \forall\, u\in C, v\in D.
\]
But this can only be true if $x^*=0$ since
\[
X=\{u+v\colon u\in C, v\in D\}.
\]
\end{proof}

\begin{thm}\label{thm:RRRconv}
The distance $\|x\|$ from the space of fixed points in the local RRR flow, for the generalized form \eqref{gamma}, is strictly decreasing for $x\ne 0$ and $\gamma>0$.
\end{thm}
\begin{proof}
Using the flow equation \eqref{xflow} in the time derivative of the squared distance,
\[
\frac{d}{dt}\,\|x\|^2=2 x^T\dot{x},
\]
and the symmetry of projections under transpose,
\[
x^T(P_D P_C-P_C P_D)x=0,
\]
we obtain
\[
\frac{d}{dt}\,\|x\|^2=-2\gamma\, Q(x),
\]
where the result follows if we can show
\[
Q(x)=x^T (P_C+P_D-P_C P_D-P_D P_C)x
\]
is a positive definite quadratic form.
From the idempotency of projections we have the identity
\begin{align*}
P_C+P_D-P_C P_D-P_D P_C&=P_C(1-P_D)P_C+(1-P_C)P_D(1-P_C)\\
&=P_C P_{D_\perp}P_C+P_{C_\perp}P_D P_{C_\perp},
\end{align*}
where the last line is expressed in terms of projections to the orthogonal complements of $C$ and $D$ in $X$. Using this identity, the quadratic form can be expressed as a sum of squares:
\[
Q(x)=\|P_{D_\perp}P_C\, x\|^2+\|P_D P_{C_\perp}\, x\|^2.
\]
To show that $Q$ has no non-trivial null vector $x^*$, let $u=P_C\, x^*$, so $u\in C$. For the first square to vanish we must have $u\in D$, and therefore $u\in C\cap D=\{0\}$. From $P_C\, x^*=u=0$ we then have $x^*\in C_\perp$. Since now $P_{C_\perp}\,x^*=x^*$, for the second square to vanish we must have $x^*\in D_\perp$. Thus both squares vanish if and only if $x^*\in C_\perp\cap D_\perp$ which, by the lemma, implies $x^*=0$.
\end{proof}

\newpage

\section{}\label{sec:projappendix}

\subsection{Bilinear constraint}

The projection to the bilinear constraint, in its simple form when used in NMF, or as generalized for deep networks, can be treated in a unified way. Most generally we seek the map $(x,w,y)\to (x',w',y')$ that minimizes
\begin{equation*}
\|x'-x\|^2+\|w'-w\|^2+g^2(y'-y)^2
\end{equation*}
subject to the constraint
\begin{equation}\label{x'w'=y'}
x'\cdot w'=\Omega\, y'.
\end{equation}
For the simple bilinear constraint we set $\Omega=1$ and take the limit $g\to \infty$, replacing the variable $y'$ in a deep network by the known data $y$ in NMF. 
Introducing a scalar Lagrange multiplier variable $u$ to impose \eqref{x'w'=y'}, we obtain the following system of linear equations,
\begin{align*}
0&=x'-x-u\, w'\\
0&=w'-w-u\, x'\\
0&=y'-y+u\,\Omega/g^2,
\end{align*}
with solution \eqref{x'w'} for $x'$ and $w'$ in the simple case and augmented by
\begin{equation*}
y'=y-u\,\Omega/g^2
\end{equation*}
for deep networks.

Imposing \eqref{x'w'=y'} on the solution, we obtain the following equation for $u$,
\begin{equation}
0=\frac{p(1+u^2)+q u}{(1-u^2)^2}-\Omega y+(\Omega/g)^2u=h_0(u),
\end{equation}
where $p$ and $q$ are the scalars in \eqref{pq}. The uniqueness of the solution for $u$ and our method for computing it is based on the following lemma:
\begin{lem}\label{hmono}
In the domain $u\in(-1,1)$, and with parameters satisfying \eqref{q>2p}, the function $h_0(u)$ is strictly increasing and has a unique zero $u_0$ and point of inflection $u_2$.
\end{lem}
\begin{proof}
In addition to $h_0$, we will need its first ($h_1$), second ($h_2$), and third ($h_3$) derivatives:
\begin{align*}
h_1(u)&=\frac{q(1+3 u^2)+2p u(3+u^2)}{(1-u^2)^3}+(\Omega/g)^2\\
h_2(u)&=3\,\frac{q u(4+4u^2)+2p(1+6 u^2+u^4)}{(1-u^2)^4}\\
h_3(u)&=12\,\frac{q(1+10 u^2+5 u^4)+2p u(5+10u^2+u^4)}{(1-u^2)^5}.
\end{align*}
Let $t_0(u)=p(1+u^2)+q u$ be the numerator of the first term in $h_0(u)$. Using \eqref{q>2p}, $t_0(-1)=-(q-2p)<0$ and $t_0(1)=q+2p>0$, so that $\lim_{u\to \pm 1}h_0(u)=\pm\infty$. Since $h_0(u)$ is continuous, its range is $(-\infty,+\infty)$. Using the same arguments we can show that exactly the same conclusion applies to $h_2(u)$. 

Now consider the numerator $t_1(u)$ of the first term in $h_1(u)$. Again using \eqref{q>2p}, and $|u|\le 1$, we have $2pu>-q |u|$ and the bound,
\begin{equation*}
t_1(u)>q+3 q u^2-q |u|(3+u^2)=q(1-|u|)^3.
\end{equation*}
This implies
\begin{equation*}
h_1(u)>\frac{q}{(1+|u|)^3}+(\Omega/g)^2>0,
\end{equation*}
and therefore $h_0(u)$ is strictly increasing. We always have a zero $u_0\in (-1,1)$ because $h_0(u)$ has range $(-\infty,+\infty)$. By the same argument we find that
\begin{equation*}
h_3(u)>\frac{12 q}{(1+|u|)^5}>0,
\end{equation*}
so that $h_2(u)$ has a unique zero, giving $h_0(u)$ a unique inflection point $u_2\in (-1,1)$.
\end{proof}

By lemma \eqref{hmono}, there is a unique root $u_0$ of $h_0(u)$ and therefore a unique projection whenever \eqref{q>2p} holds. The other properties of $h_0(u)$ motivate the following two-mode algorithm for finding $u_0$.

Start with $u_a=0$ as the ``active bound" on $u_0$; this will be the base point for a Newton iteration. Depending on the sign of $h_0(u_a)$, $u_b=\pm 1$ will be the initial ``bracketing bound" on $u_0$. From the Newton update
\begin{equation}
u'=u_a-\frac{h_0(u_a)}{h_1(u_a)},
\end{equation}
we take one of two possible actions. If $u'$ is in the interval bracketed by $u_b$, we set $u'_a=u'$ and reset the bracketing bound $u_b'=u_a$ if the sign of $h_0(u_a')$ has changed (keeping $u_b'=u_b$  otherwise). If, on the other hand, $u'$ is outside the interval bracketed by $u_b$, the new active bound is obtained by bisection, $u_a'=(u_a+u_b)/2$, and $u_b'$ is set to either of the previous bounds, $u_a$ or $u_b$, depending on the sign of $h_0(u_a')$.

By taking either a Newton step or a bisection step, the interval bracketing $u_0$ is made smaller. By the lemma's unique inflection point $u_2$, eventually $h_2$ will have the same sign at both endpoints of the interval. The function $h_0$ is now convex/concave on the interval and all subsequent iterations always take the Newton step, converging quadratically to the root $u_0$. The case $u_0=u_2$ presents an exception, but the convergence by bisection steps will still be linear.

\newpage

\section{}\label{sec:binarycodeappendix}

\subsection{Binary encoding with continuous weights and step activation}

It is a straightforward exercise to completely characterize the weights and biases that solve the binary encoding problem of section \ref{sec:binaryencoding} when the network is trained with the step-activation constraint shown in Figure \ref{fig:fig2}. Combinatorially there are $(2^n)!$ solutions, corresponding to how the 1-hot positions are mapped to the integers $0,1,\ldots, 2^n-1$. Consider one such solution and let $j\in C$ be the code node that codes a particular bit, and $D_1(j)\subset D$ be the corresponding 1-hot positions/integers that have a 1 in their binary representation for that bit. Let $D_0(j)$ be the complement of $D_1(j)$, that is, the subset of input nodes which are assigned a 0 for bit $j$. If $\Delta$ is the gap in the step-activation, and neglecting the weight-normalization constraint for now, a necessary and sufficient set of constraints on the parameters for correct encoding is
\begin{subequations}\label{binrep0}
\begin{align}
\forall j\in C,\;\;\forall i\in D_1(j):\quad &w[i\to j]/\Omega-b[j]\ge\Delta/2\\
\forall j\in C,\;\;\forall i\in D_0(j):\quad &w[i\to j]/\Omega-b[j]\le -\Delta/2.
\end{align}
\end{subequations}
Combining these to eliminate the biases we obtain
\begin{equation}\label{binrep1}
\forall j\in C,\;\;\forall i\in D_1(j),\;\;\forall i'\in D_0(j):\quad w[i\to j]-w[i'\to j]\ge \Omega\,\Delta.
\end{equation}
Now define
\begin{subequations}
\begin{align}
w_+(j)&=\min_{i\in D_1(j)}w[i\to j]\\
w_-(j)&=\min_{i\in D_0(j)}-w[i\to j].
\end{align}
\end{subequations}
Supposing our weights satisfy \eqref{binrep1}, then
\begin{equation}
\forall j\in C:\quad w_+(j)+w_-(j)\ge \Omega\,\Delta
\end{equation}
and this guarantees that the constraints on the biases from \eqref{binrep0}
\begin{equation}
\forall j\in C:\quad  -w_-(j)/\Omega+\Delta/2\le b[j]\le w_+(j)/\Omega-\Delta/2
\end{equation}
always has a solution.

When the weights into node $j$ have norm $\Omega$, the inequalities \eqref{binrep1} will not have a solution when $\Delta$ is too large. To obtain the precise limit we use the following:
\begin{lem}\label{gaplem}
Suppose $(x,y)\in \mathbb{R}^2$ satisfy $x-y\ge a> 0$. Then $x^2+y^2\ge a^2/2$ and the equality case corresponds to $(x,y)=(a/2,-a/2)$.
\end{lem}
\begin{proof}
The minimum squared distance to the half-plane constraint is $a^2/2$ and is uniquely attained for the stated assignment.
\end{proof}
Consider an arbitrary matching of the nodes in $D_1(j)$ with the nodes in $D_0(j)$, and the corresponding $|D|/2$ instances of lemma \ref{gaplem} in the constraints \eqref{binrep1}. Additively combining the resulting norm inequalities we obtain
\begin{equation}
\Omega^2\ge (|D|/2)(\Omega\,\Delta)^2/2,
\end{equation}
or
\begin{equation}
\Delta\le \frac{2}{\sqrt{|D|}}.
\end{equation}
We only get equality when all $|D|/2$ inequalities of the matching are equalities, and for that case the lemma specifies a unique solution:
\begin{subequations}
\begin{align}
\forall j\in C,\;\;\forall i\in D_1(j):\quad &w[i\to j]=\frac{\Omega}{\sqrt{|D|}},\\
\forall j\in C,\;\;\forall i\in D_0(j):\quad &w[i\to j]=-\frac{\Omega}{\sqrt{|D|}}.
\end{align}
\end{subequations}

The analysis of the decoder is similar. For any $i\in D$ let $C_1(i)\subset C$ be the code nodes on which the corresponding integer assigned to $i$ has a 1 in its binary representation. For the same integer $i$ the nodes $C_0(i)$ in the complement have a 0 bit. Now define
\begin{equation}
\forall i\in D:\quad w_1(i)=\sum_{j\in C_1(i)} w[j\to i].
\end{equation}
The necessary and sufficient set of constraints on the parameters is now
\begin{subequations}\label{binrep2}
\begin{align}
\forall i\in D:\quad &w_1(i)/\Omega -b[i]\ge \Delta/2\\
\forall i\in D,\;\; \forall j\in C_1(i):\quad &(w_1(i)-w[j\to i])/\Omega -b[i]\le -\Delta/2\\
\forall i\in D,\;\; \forall j\in C_0(i):\quad &(w_1(i)+w[j\to i])/\Omega -b[i]\le -\Delta/2
\end{align}
\end{subequations}
where the last two inequalities cover, respectively, the case of a correct 1 bit flipping to 0 and a correct 0 bit flipping to 1. Comparing these inequalities with the first we infer
\begin{subequations}\label{binrep3}
\begin{align}
\forall i\in D,\;\;\forall j\in C_1(i):\quad &w[j\to i]\ge \Omega\,\Delta\\
\forall i\in D,\;\;\forall j\in C_0(i):\quad &w[j\to i]\le -\Omega\,\Delta.
\end{align}
\end{subequations}
When these inequalities are satisfied we can always find biases that satisfy \eqref{binrep2}. Moreover, since the norm of the weights into node $i$ is $\Omega$, from \eqref{binrep3} we obtain the inequality
\begin{equation}
\Omega^2\ge |C| (\Omega\,\Delta)^2
\end{equation}
or
\begin{equation}
\Delta\le\frac{1}{\sqrt{|C|}}.
\end{equation}
The equality case corresponds to only equalities in \eqref{binrep3}, that is, weights differing only in sign as dictated by membership of $j$ in $C_1(i)$ or $C_0(i)$.

\newpage

\bibliography{lwl}

\end{document}